\documentclass[11pt]{article}
\usepackage[utf8]{inputenc}

\usepackage{graphicx,subfigure}
\usepackage{amsthm,amsmath,amssymb}
\usepackage{fullpage}
\usepackage{hyperref}
\usepackage{authblk}

\usepackage{xcolor}

\newtheorem{lemma}{Lemma}
\newtheorem{theorem}{Theorem}

\title{On Goodhart's law, with an application to value alignment}

\author[1,2]{El-Mahdi El-Mhamdi~\footnote{Part of this work was done while the author was also at Google}}
\author[2,3]{Lê-Nguyên Hoang}
\affil[1]{École Polytechnique, France}
\affil[2]{Calicarpa, Switzerland}
\affil[3]{Tournesol Association, Switzerland}

\begin{document}

\maketitle

\begin{abstract}
``When a measure becomes a target, it ceases to be a good measure'', this adage is known as {\it Goodhart's law}. In this paper, we investigate formally this law and prove that it critically depends on the tail distribution of the discrepancy between the true goal and the measure that is optimized. Discrepancies with long-tail distributions favor a Goodhart's law, that is, the optimization of the measure can have a counter-productive effect on the goal. 

We provide a formal setting to assess Goodhart's law by studying the asymptotic behavior of the correlation between the goal and the measure, as the measure is optimized. Moreover, we introduce a distinction between a {\it weak} Goodhart's law, when over-optimizing the metric is useless for the true goal, and a {\it strong} Goodhart's law, when over-optimizing the metric is harmful for the true goal. A distinction which we prove to depend on the tail distribution.

We stress the implications of this result to large-scale decision making and policies that are (and have to be) based on metrics, and propose numerous research directions to better assess the safety of such policies in general, and to the particularly concerning case where these policies are automated with algorithms.
\end{abstract}
\section{Introduction}

``When a measure becomes a target, it ceases to be a good measure''. This quote by anthropologist Marilyn Strathern~\cite{strathern1997} is probably the most commonly known rephrasing of {\it Goodhart's law}, which originally stated that {\it ``any observed statistical regularity will tend to collapse once pressure is placed upon it for control purposes''} \cite{goodhart1975}. This principle is critical in numerous areas where metrics are used as proxies for evaluation and decision making, from economics to management, public policy, peer-reviewing of scientific papers and grants, education, standard tests for school admission, content recommendation on social networks and the list could be as long and interdisciplinary as one could think of metrics-based decisions. Variants of this adage could also be found in popular wisdom and proverbs around the world, often describing situations where doing good according to some \emph{perceived} form of the intention, leads to harm according to the \emph{real} intention.

Goodhart's law turns out to be critical for statistical learning as well. Indeed, {\it overfitting\footnote{Overfitting refers to situations where a model is excessively fitting the data used to train it, to the point where it becomes counterproductive and leads to no being able to reproduce its performance on data not used during training.}} can be regarded as an instance of Goodhart's law, as well as numerous metrics like {\it accuracy} in standard machine learning evaluations. This problem has long been considered settled by the VC-theory \cite{hastie2009,shalev2014}. However, recently, empirical results question the relevancy of this theory \cite{zhang2017,belkin_hsu2019,nakkiran2019}, especially as deep neural networks achieve state-of-the-art generalization despite over-parameterization and interpolation of the training data.

Moreover, there is a growing interest in the {\it alignment problem} \cite{noothigattu2018,russell2019,hadfield2019, hoang2019fabuleux}, that is, assigning to decision making agents an objective function that really captures what we desire to optimize. But our preferences are usually too vague and complex to be formalized. We often replace them by proxies, which are mostly correlated with what we desire. But Goodhart's law argues that optimizing these proxies might actually cause harm.

To better understand the risk of counter-productive optimization of an approximate measure, we analyze Goodhart's law in a very basic model, where $M = G+\xi$. In other words, the measure $M$ that we maximize deviates from the ``true'' goal $G$ by some (small) discrepancy $\xi$. Assume we maximize $M$, what values of $G$ do we then obtain? Do we still have a positive correlation between $M$ and $G$ among top values of $M$?

In this paper, we show that this strongly depends on the tail distribution of the discrepancy $\xi$, in the case where $\xi$ and $G$ are independent. In essence, we show that for normally distributed discrepancies, we only have a very \emph{weak Goodhart's law}. While the correlation goes to zero, it does so extremely slowly. However, for distribution with subexponential tail distribution, Goodhart's law may become strong. In fact, if the discrepancy has a power law distribution thicker than that of the goal, then we obtain a \emph{strong Goodhart's law}: there is a negative correlation between $M$ and $G$ among top values of $M$. Or put differently, maximizing $M$ then becomes harmful to the goal $G$.

Worryingly, we also show that the maximization of $M$ can be infinitely harmful to the goal $G$, assuming that the tail distribution of the discrepancy $\xi$ can depend on the value of the goal $G$. Intuitively, as $M$ is over-optimized, $G$ will take values that maximize the thickness of the tail distribution of $\xi$. If this tail is thickest for infinitely negative values of $G$, over-optimizating $M$ leads to \emph{infinitely negative values} of $G$.

Perhaps most importantly, our results suggest that the strong Goodhart's law may apply to essentially any algorithm (even without machine learning!) deployed in complex environments that are modified by the algorithm. In particular, we provide a very simplified simulation of an interaction between a user and an algorithm, where the recommendations of the algorithm affect the user's preferences. We show that, even in this very basic setting, a strong Goodhart's law can manifest itself. Thereby, we highlight a major flaw in the testing of algorithms on fixed datasets. Arguably, such tests are very insufficient, as they are blind to nonlinear dynamics that may expose algorithms to the strong Goodhart's law once deployed. 

Back to the historic roots of Goodhart's law, we finally argue that our results are relevant in any context where a metric is used as a proxy for a more complex goal, be it in public policy, education or as initially formulated, in economics.

\subsection*{Related works}

{\bf Goodhart's law.} Goodhart's law was introduced in \cite{goodhart1975}, but it is closely related to Lucas' concurrent critique of macroeconomic models \cite{lucas1976} (see \cite{chrystal2003}). It is also often regarded as a variant of Campbell's law \cite{campbell1979}. Recently, it was further analyzed by \cite{manheim2018} and \cite{demski2019}, who seeked to distinguish different possible root causes of Goodhart's law. However, to the best of our knowledge, our present work is the first to establish mathematically a direct link between Goodhart's law and the tail distribution of the difference between the measure and the goal.

{\bf Overfitting.} Our approach provides a new analysis of the otherwise widely studied problem of overfitting. Interestingly, classical statistical learning \cite{hastie2009,shalev2014} has been argued to fall short of explaining the behavior of modern machine learning. While classical statistical learning argued that one should avoid invoking complex models to fit limited datasets, today's deep neural networks \cite{zhang2017,belkin_hsu2019,nakkiran2019}, kernel methods \cite{belkin2018,belkin_rakhlin2019}, but also ridgeless (random feature) linear regression \cite{muthukumar2019,bartlett2019,mei2019,hastie2019,belkin_hsu_xu2019} and even ensembles \cite{belkin_hsu2019}, usually achieve their best performances by massive overparameterization and perfect data fitting (called {\it interpolation}). 

Intriguingly, many machine learning algorithms feature so-called {\it double descent} \cite{belkin_hsu2019,nakkiran2019}. This phenomenon is captured by a curve where, as training loss decreases, out-of-sample loss roller-coasters, by first decreasing, then increasing and then decreasing again. This phenomenon is not predicted by classical statistical learning. As argued by \cite{zhang2017}, ``understanding deep learning requires rethinking generalization'', that is, the capacity of a model that fits training data to generalize to out-of-sample data. 

While the discussions of our paper also fall short from an explanation of overfitting, we hope to lay some groundwork towards a novel approach to analyze, understand and foresee the potential counter-productive over-optimization of sample loss.

{\bf Alignment.} It has been argued that, as machine learning algorithms are being deployed on massive scales, it has become urgent to make sure that what they optimize reflects adequately what we would want them to optimize. This is called the {\it alignment problem} \cite{noothigattu2018, freedman2018, hoang2019fabuleux, russell2019, hadfield2019}. Typically, algorithms should not be blind to the {\it side effects} caused by their deployment \cite{amodei2016}.

One large-scale example of alignment problems is that of recommender systems of large internet platforms such as Facebook, YouTube~\footnote{This paper was stalled, when one of the authors worked at Google, because of this very mention of YouTube. Some of the main reviewers handling ``sensitive topics'' at Google were ok with this paragraph and the rest of the paper, as long as it mentioned other platforms but not YouTube, even if the problem is exactly the same and is of obvious public interest.}, Twitter, TikTok, Netflix and the alike. So far, these platforms mostly optimize for click-through rate, watch-time or similar measures often under the umbrella of ``user engagement''~\cite{facebook0, facebook1, facebook2}. Such a measure might have seemed to be correlated with how we would want these systems to behave. But by over-optimizing it, click-through rate maximizers have been argued to cause numerous undesirable {\it side effects}, like promoting misinformation \cite{allgaier2019}, polarization \cite{ribeiro2019} and anger \cite{berger2012}. Theoretical work also suggests the emergence of unforeseen and perhaps undesirable features \cite{TSST19}. This is arguably no longer what the algorithm designers would want their algorithms to do. 

Interestingly, the impact of the deployment of influential algorithms has been studied in a very different field, namely finance. In fact, it is well-known that any market decision changes the market, which is known as {\it market impact} \cite{almgren2005}. Unfortunately, it was suggested that such deployments are often accompanied with unstable long-tail distributed fluctuations \cite{thurner2012}. Such distributions, which we shall argue to be a major concern in terms of (strong) Goodhart's law, are argued to be ubiquitous, in natural language processing \cite{zipf1949}, scale-free networks \cite{dorogovtsev2000} and economics \cite{brynjolfsson2010}. This suggests that the alignment problem may be significantly harder than expected.

\subsection*{Structure of the paper}

The rest of the paper is organized as follows. 
Section~\ref{sec:model} introduces our model, outlines our analysis and then informally presents our key findings. 
Section~\ref{sec:results} then provides our formalization of Goodhart's law and details our results. Once our formalism is set, the proofs are mostly consisting in direct derivations and are given in the Appendix, while useful qualitative analysis of the theorems or their proofs are kept in the main paper.
Section~\ref{sec:alignment} explores the consequences of our results on the alignment problem, and provides an example to highlight their applicability in practical scenarios, such as content recommendation.
Section~\ref{sec:remarks} provides additional remarks and avenues for future research.
Section~\ref{sec:conclusion} then concludes.

\section{Model}
\label{sec:model}

Consider a large set $\Omega$ of proposals. These can be research papers, strategic plans or weights of a neural network. For any $\omega \in \Omega$, we assume that there is a ``true'' score $G(\omega)$. The true goal is thus to maximize $G(\omega)$ over $\omega \in \Omega$. However the $G$ function may be unusable. Perhaps it is too long to compute, or to describe in a Turing machine, in particular in human understandable language. Perhaps it is also unknown. 

As an example, it is arguably extremely challenging to formalise quantitatively what {\it ``being a good scientist''} is. In practice, too often, instead of optimizing for this goal $G$, we resort to proxy measures such as ``impact of scientific work'', which in turn require even more proxies $M$ such as ``number of citations'' or other controversial indexes and rankings.

As another example, consider the ``true'' objective function of a supervised learning problem. The problem is then typically to learn the parameters $\omega \in \Omega$ of a statistical model (e.g. the parameters of a linear regression or the weights of a neural network) so as to maximize accuracy on out-of-sample data. This will usually be of the form $G(\omega) = \mathbb E_{x \leftarrow \mathcal D}[-\mathcal L(x,\omega)]$, where $x$ is the data drawn from an unknown distribution $\mathcal D$ and $\mathcal L$ is a data-dependent loss function\footnote{Typically, if $f_\omega$ is the function computed by a statistical model of parameters $\omega$, and if $(x,y)$ is a feature-label pair, then we may define $\mathcal L((x,y),\omega) = (f_\omega(x)-y)^2$.}. Instead, in practice, we then turn to measures $M$ that approximate $G$. Typically, in the supervised learning example, we may replace $\mathcal D$ by some sample $S$, which is usually known as the training set. Then, our measure of the performance of the model of parameters $\omega \in \Omega$ is $M_S(\omega) = - \frac{1}{|S|} \sum_{x \in S} \mathcal L(x,\omega)$.

In such a case, we are very likely going to construct a measure $M$ that is correlated with the goal $G$. The natural way to formalize this correlation is by first introducing a probability distribution $\tilde \Omega$ over $\Omega$. In the case of supervised learning, this would correspond to considering random machine learning models, and to determine whether, for such random models, $M$ and $G$ are correlated. Note, however, that the probability distribution over models can be quite sophisticated, e.g. neural networks obtained by random initialization then optimized after some 100,000 iterations of stochastic gradient descent.

Given the probability distribution $\tilde \Omega$, we can define expectations such as $\mathbb E[M] = \mathbb E_{\omega \leftarrow \tilde \Omega}[M(\omega)]$, as well as the correlation between $G$ and $M$ by $\rho = Cov(M,G) / \sqrt{Var(M)Var(G)}$.

Let $\xi(\omega) = M(\omega) - G(\omega)$ be the discrepancy between the measure and the goal. In the sequel, for simplicity, we shall make the assumption that this discrepancy $\xi$ is independent from the goal $G$ (except in Section \ref{sec:worst_case} where we study a worst-case scenario). Denoting $\varepsilon^2 = Var(\xi) / Var(G)$ the noise-to-signal ratio, we can then write 
$\rho = \frac{1}{\sqrt{1+\varepsilon^2}}$.

 In particular, the correlation goes to 1 as the noise-to-signal ratio $\varepsilon = \sqrt{\rho^{-2}-1}$ goes to zero. In such a case, the measure $M$ seems to be a good proxy of the goal $G$.

\subsection{When the measure becomes a target}

The trouble though occurs when $M$ becomes a target. In particular, when this is done, because of selection bias, there is now a focus on the choices of $\omega$ for which $M$ is large. In particular, if most options are discarded except for, e.g. the top $\alpha = 1\%$ values of $M$, then the correlation between $M$ and $G$ within this biased sample may no longer be large. In particular, if this correlation becomes zero or negative, then the measure $M$ ceases to be a good measure.

To formalize, let $m_\alpha$ be the top $\alpha$ quantile for $M$, i.e. such that\footnote{For simplicity, we only consider distributions with probability density functions.} $\mathbb P[M\geq m_\alpha] = \alpha$. As a shorthand, whenever some probability, expectation, variance or covariance is conditioned on $M\geq m_\alpha$, we shall use an $\alpha$ subscript, i.e. we write $\mathbb P_\alpha[E] = \mathbb P[E | M\geq m_\alpha]$ and $\mathbb E_\alpha[X] = \mathbb E[X|M \geq m_\alpha]$.

We denote by $\rho_\alpha$ the correlation between $M$ and $G$ for the top $\alpha$ quantile for $M$, i.e.

$$\rho_\alpha = \frac{Cov_\alpha(G,M)}{\sqrt{Var_\alpha(G) Var_\alpha(M)}}.$$

In this paper, we analyze Goodhart's law by focusing on how $\rho_\alpha$ behaves as $\alpha \rightarrow 0$. Evidently this depends on the distribution of $G$ and $\xi$. But as we shall see, for reasonable distributions of $G$ and $\xi$, the correlation $\rho_\alpha$ can be near-zero (which we call \emph{weak Goodhart}) or negative (which we call \emph{strong Goodhart}). We even show that, as $\alpha \rightarrow 0$, the expected value of the goal can become infinitely negative (which we call \emph{infinite Goodhart}).

In fact, the behavior of $\rho_\alpha$ greatly depends on the tail distributions of $G$ and $\xi$. In particular, how do the functions $\mathbb P[G \geq g]$ and $\mathbb P[\xi \geq x]$ behave as $g$ and $x$ get larger?

\subsection{Key finding}

The key takeaway of this paper is that Goodhart's law strongly depends on the large deviation behavior of the discrepancy $\xi$ between the measure $M$ and the goal $G$. In particular, we show that we only have a very weak version of Goodhart's law (the correlation goes to zero very slowly as the proxy measure is maximized) when $\xi$ follows a normal distribution.

However, for thicker tail distribution of the discrepancy, if $\alpha \approx \varepsilon$, then Goodhart's law can apply. In fact, especially if the discrepancy follows a power law distribution, the measure $M$ can become negatively correlated to $G$. In such cases, it may become critical to stop maximizing $M$, as over-optimizing $M$ would then be detrimental to $G$. 

Yet recall that $\varepsilon = \sqrt{\rho^{-2}-1}$. This thus suggests that it can be counter-productive to select a fraction of top candidates smaller than $\alpha = \sqrt{\rho^{-2}-1}$. Disturbingly, if $\rho = 0.9$, then $\alpha \approx 49\%$, meaning that we should not be more selective than merely picking up the top 49\%. Even with $\rho = 0.99$, we should not have $\alpha < 14\%$. And to safely select the top $\alpha = 1\%$, we need a correlation between the goal and the metric that is as large as $\rho = 0.99995$.

In this paper, we also show the possibility of a catastrophic worst-case Goodhart effect. More precisely, we present a setting where over-optimizing the measure $M$ can be infinitely harmful to the goal, i.e. $\mathbb E_\alpha[G] \rightarrow -\infty$, even though the mean of the discrepancy is zero and its variance is $\varepsilon^2$ for any value of the goal $G$. The key feature of this setting is that the long-tailness of $\xi$ is different for different values of $G$ (in particular, here, $G$ and $\xi$ are no longer assumed to be independent).


Perhaps even more importantly, Goodhart's law has strong implications for the safety of large-scale algorithms capable of modifying their environments, such as recommendation systems deployed on social networks that can create addictive~\footnote{We define an addictive use of a social media following~\cite{turel2018time} as ``a state in which the individual is overly concerned about social media activities, driven by an uncontrollable motivation to perform the behavior, and devotes so much time and effort to it, so that it interferes with other important life areas''.} behaviors~\cite{TurelBB18, HawiSamaha16, MarkICJS16, EichstaedtSMUC+18} on potentially billions of users on which they are deployed. Indeed, in such a case, feedback loops between the algorithm and its environment might cause~\footnote{We illustrate this in Section~\ref{sec:illustrative} through a simple model.} heavy tails~\cite{cui2012emergence, goetz2009modeling} for the discrepancy between testing and deployment performances. If so, the over-optimization of testing metrics might turn out to select dangerously sub-optimal algorithms. 






\section{Formal results}
\label{sec:results}

\subsection{Bounded goal, exponential discrepancy}

We start with the easiest case, namely when the goal $G$ is uniform on $[0,1]$, and when $\xi$ follows some exponential distribution of parameter $\lambda$. In other words, in this subsection, we assume that the probability density functions of $G$ and $\xi$ are given by $p^G(g) = 1$ for $g \in [0,1]$, and $p^\xi(x) = \lambda e^{-\lambda x}$. Given that $Var(G) = 1$, to guarantee a signal-to-noise ratio equal to $\varepsilon$, we must then have $\lambda = 1/\varepsilon\sqrt{12}$.

Recall the memoryless property of the exponential distribution. In particular, it implies that $\mathbb E [\xi | \xi \geq x] = x + \mathbb E[\xi] = x+1/\lambda$. This memoryless property is why, as we overoptimize over $M$, nothing changes. More precisely, after some threshold for $\alpha$, the correlation between $G$ and $M$ becomes zero, and the expected value of $G$ given that we are within the top $\alpha$ values of $M$ is constant. This constant can be as low as $1/2$ for very large discrepancy $\varepsilon$. And it remains of the order $1-O(\varepsilon)$ for small discrepancy $\varepsilon \rightarrow 0$. All of this is formalized by the following theorem.

\begin{theorem}
If $G$ is uniform on $[0,1]$ and $\xi$ follows an exponential distribution such that the noise-to-signal ratio is $\varepsilon$, then for $0 < \alpha \leq \varepsilon \sqrt{12} (1-e^{-1/\varepsilon\sqrt{12}})$, we have $\rho_\alpha = 0$ and $\mathbb E_\alpha[G]  = \frac{1}{\lambda} \frac{\lambda e^\lambda - e^\lambda + 1}{e^\lambda-1}$. As $\varepsilon \rightarrow 0$, we have $\mathbb E_\alpha [G] = 1-\varepsilon\sqrt{12} +\mathcal O(\varepsilon^2)$.
\label{th:uniform_exp}
\end{theorem}

\begin{proof}
The proof of this result is derived in Appendix \ref{app:uniform_exponential}. 
\end{proof}

\begin{figure}[!ht]
        \includegraphics[width=0.5\textwidth]{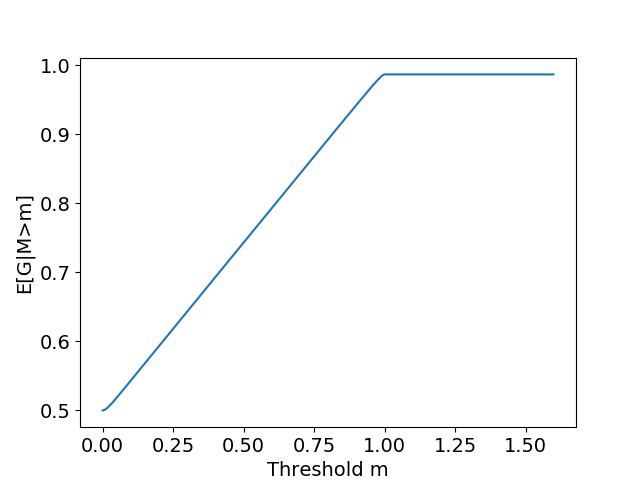} 
        \includegraphics[width=0.5\textwidth]{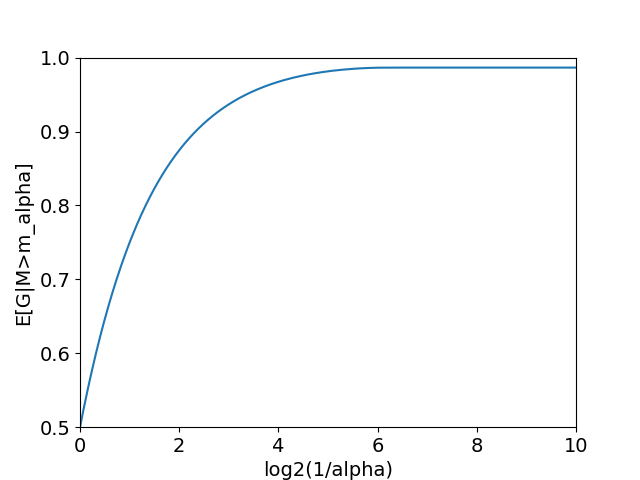} 
\caption{As $\alpha$ decreases, the measure threshold $m_\alpha$ increases. At some point, the expected goal $\mathbb E_\alpha[G]$ reaches a plateaus. The figure depictes expected goal as a function of the measure threhold (on the left) and as a function of $\log_2(1/\alpha)$ (on the right). The parameters were set at $\varepsilon = 1/256$.}
\label{fig:G_given_M_uniform_exp}
\end{figure}

Figure \ref{fig:G_given_M_uniform_exp} depicts the expected value of the goal $G$ given the measure $M$. As we can see, the optimization of $M$ initially yields improvement of $G$. But at some point, even though $M$ keeps increasing, $G$ stops to do so. The correlation between $M$ and $G$ is then nil. The measure literally ceases to be a good measure, it does not, however harm (decrease) the goal to keep optimizing $M$. This is what we call a \emph{weak Goodhart} situation.

\subsection{Normal distributions}

In this subsection, we assume that $G$ and $\xi$ follow normal distributions. Without loss of generality, we assume $G \leftarrow \mathcal N(0,1)$ and $\xi \leftarrow \mathcal N(0,\varepsilon^2)$. As a result, we know that $M \leftarrow \mathcal N(0, 1+\varepsilon^2)$. We shall denote $\tau = 6.283...$ the ratio of the perimeter of a circle to its radius.

The main result here is a very weak Goodhart's law. Namely, correlation goes to zero. But it does so extremely slowly.

\begin{theorem}
\label{th:normal}
If $G$ and $\xi$ follow normal distributions, then 
\begin{equation}
\rho_\alpha \sim \frac{1}{\varepsilon \sqrt{2\ln \frac{1}{\alpha} - \ln \ln \frac{1}{\alpha^2} - \ln \tau}}.
\end{equation}
In particular, $\rho_\alpha \rightarrow 0$ as $\alpha \rightarrow 0$.
\end{theorem}

\begin{proof}
Like for the other theorems to follow, the proof of this result is slightly technical and long. It involves deriving numerous asymptotic approximations of integrals. See Appendix \ref{app:normal} for details. 
\end{proof}

\begin{figure}[!ht]
        \includegraphics[width=0.5\textwidth]{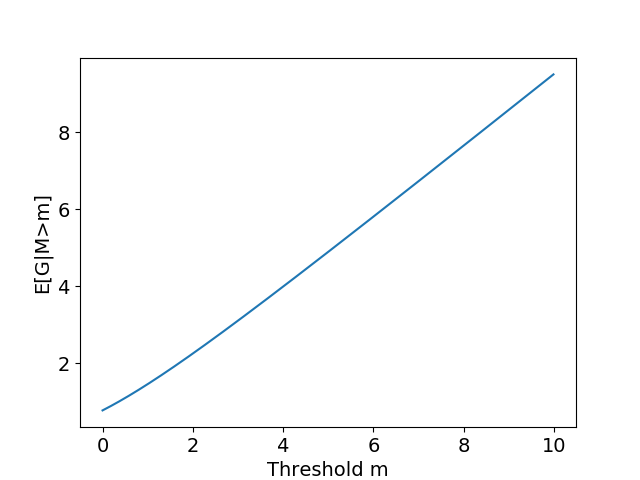} 
        \includegraphics[width=0.5\textwidth]{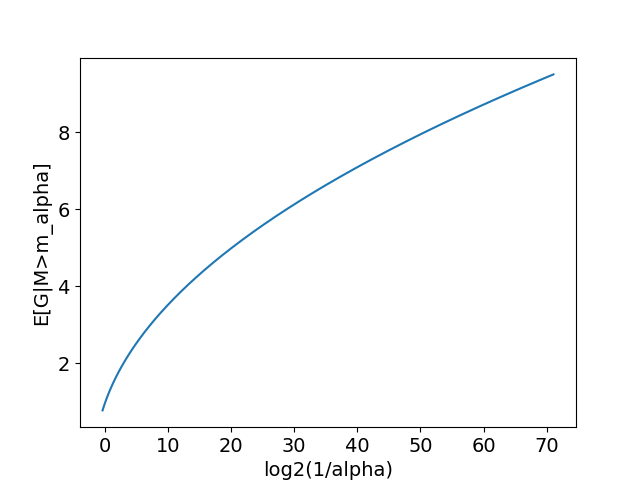} 
\caption{As $\alpha$ decreases, the measure threshold $m_\alpha$ increases. The expected goal $\mathbb E_\alpha[G]$ increases steadily. The figure depicts expected goal as a function of the measure threshold (on the left) and as a function of $\log_2(1/\alpha)$ (on the right). The parameters were set at $\varepsilon = 1/256$.}
\label{fig:G_given_M_exp_exp}
\end{figure}

In particular, we are essentially fine (we have a positive correlation between $G$ and $M$ on the top $\alpha$ quantile for $M$) so long as, roughly, $\alpha > \exp \left(- \frac{1}{2\varepsilon^2}\right)$. In such a situation, Goodhart's law does not hold.

\subsection{Bounded goal, power law discrepancy}

We now consider the case where $G$ is uniform, but where $\xi$ is a power law distribution of parameters $(\beta, \eta)$ with $\beta > 3$ (the cases $\beta \leq 3$ do not yield finite correlations because of infinite variance). For technical reasons, we shall also assume $\beta \neq 4$. More precisely, we write $p^\xi(x) = (\beta-1) \eta^{\beta-1} x^{-\beta}$ for $x \geq \eta$. This ensures that $\int_\eta^\infty p^\xi = 1$. The variance of $\xi$ is then computed by
\begin{align*}
Var(\xi)  & =  \int_\eta^\infty x^2 p^\xi(x)dx - \left( \int_\eta^\infty x p^\xi(x)dx \right)^2 \\ & = \frac{\beta-1}{(\beta-2)^2 (\beta-3)} \eta^2,
\end{align*}
which means that, for $\eta = (\beta-2) \sqrt{\frac{\beta-3}{\beta-1}} \varepsilon \sqrt{12}$, we achieve a noise-to-signal ratio equal to $\varepsilon$.

\begin{theorem}
If $G$ is uniform on $[0,1]$ and $\xi$ follows a power law distribution of decay power $\beta >3$, then $\rho_\alpha \sim - \frac{\alpha^{\frac{1}{\beta-1}}}{\sqrt{12} (\beta-2) \varepsilon}$.
\label{th:uniform_power}
\end{theorem}

\begin{proof}
See Appendix \ref{app:uniform_power}. 
\end{proof}

In particular, this theorem says that for small values of $\alpha$, we have a strong Goodhart's law as long as $\alpha$ is small, and at least of the order of $\varepsilon^{\beta-1}$. In fact the following theorem adds the fact that this strong Goodhart's law holds for $\alpha$ between roughly $\varepsilon$ and $\varepsilon^{\beta-1}$. In particular, here, we consider the case where $\alpha \approx \sqrt{\frac{\beta-1}{\beta-3}} \varepsilon$ for asymptotically small values of $\varepsilon$.

\begin{theorem}
Suppose that the threshold $\alpha(\varepsilon)$ is chosen such that $m_{\alpha(\varepsilon)} = 1+\eta = 1+(\beta-2) \sqrt{\frac{\beta-3}{\beta-1}} \varepsilon$. Then, as $\varepsilon \rightarrow 0$, we have 
$$\alpha \sim \sqrt{\frac{\beta-1}{\beta-3}} \varepsilon, \text{ and } \rho_\alpha \rightarrow \max \left\{ - \sqrt{\frac{\beta-3}{2(\beta-2)}}, - \frac{1}{\beta-2} \right\}.$$
\label{th:alpha_epsilon_uniform_power}
\end{theorem}

\begin{proof}
See Appendix \ref{app:uniform_power}. 
\end{proof}

\begin{figure}[!ht]
        \includegraphics[width=0.5\textwidth]{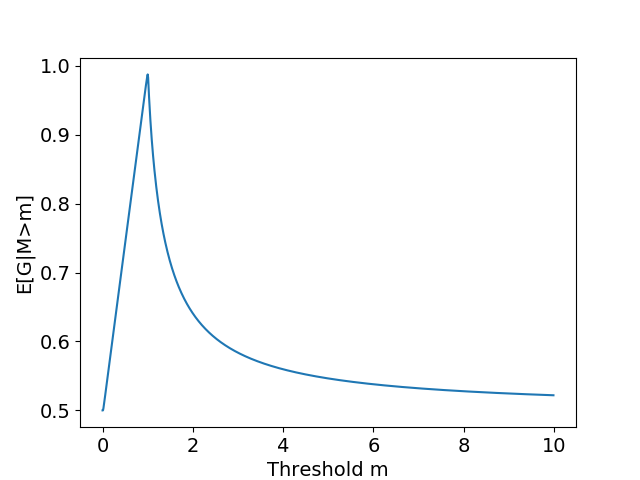} 
        \includegraphics[width=0.5\textwidth]{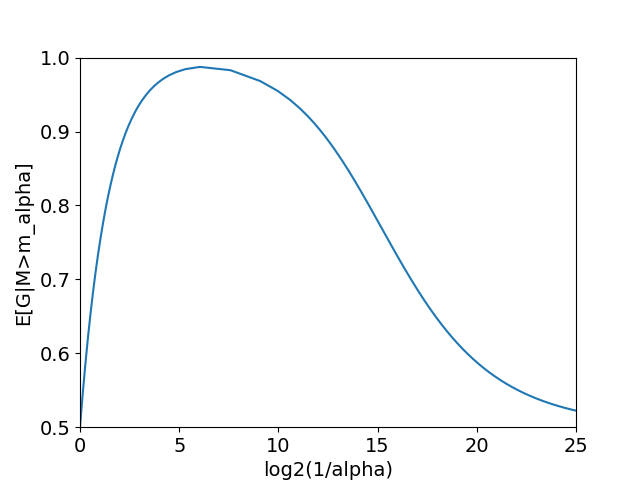} 
\caption{For a uniform goal and a power law noise, as $\alpha$ decreases, the measure threshold $m_\alpha$ increases. The expected value of the goal increases, but then decreases sharply. The figure illustrates this phenomenon as a function of the measure threshold (on the left) and as a function of $\log_2(1/\alpha)$ (on the right). The parameters were set at $\varepsilon = 1/256$ and $\beta = 3.5$.}
\label{fig:G_given_M_uniform_power}
\end{figure}

\begin{figure}[!ht]
        \includegraphics[width=0.5\textwidth]{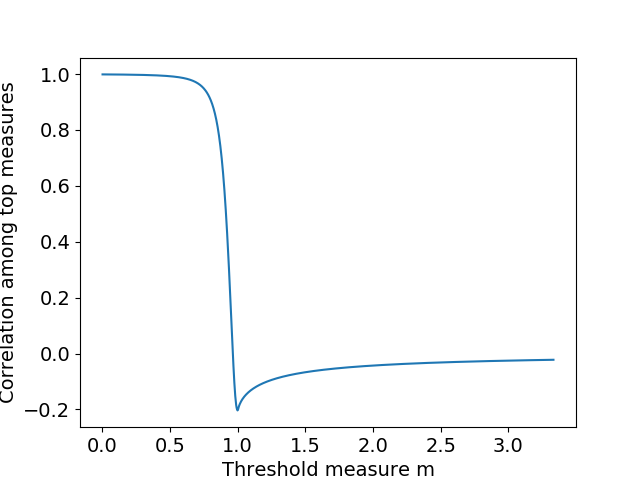} 
        \includegraphics[width=0.5\textwidth]{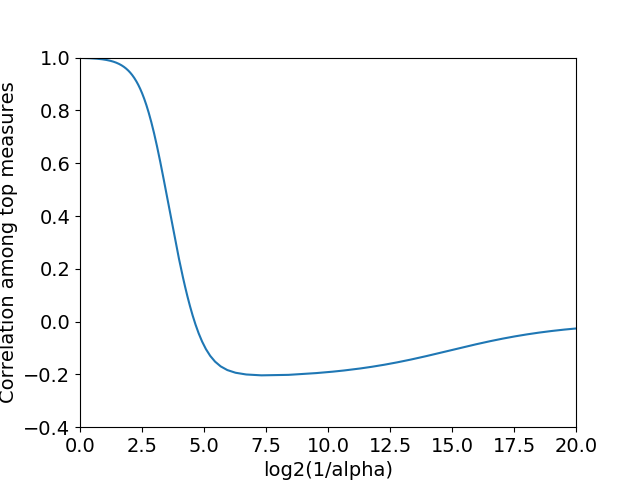} 
\caption{Consider a uniform goal and a power law noise. Initially near 1, the correlation between $M$ and $G$ decreases when $\alpha$ decreases, as a function of the measure threshold $m_\alpha$ (on the left) and of $\log_2(1/\alpha)$ (on the right). It even becomes negative, before reaching zero eventually. The plot depicts the case where $\varepsilon = 1/256$ and $\beta=3.5$.}
\label{fig:correlation_uniform_power}
\end{figure}

Why does the correlation become negative? An intuitive explanation behind this phenomenon lies in a well-known property of power law distributions. The expected value of a power-law noise $\xi$ given that $\xi$ is at least $x$ is proportional to $x$. More precisely, in our case, we have $\mathbb E[\xi|\xi \geq x] = \frac{\beta-1}{\beta-2} x = x + \frac{x}{\beta-2}$. In other words, given a noise at least $x$, we expect it to be about $\frac{x}{\beta-2}$ larger than $x$. Now, given $M \geq m$, for a large value of $m$, then we know that the noise is at least $\xi \geq m-G$. Assuming in addition $G=g$ and $m-g \geq \eta$, then the expected value of the noise is $\mathbb E[\xi|\xi \geq m-g] = m-g+ \frac{m-g}{\beta-2}$. Thus, for $m \geq \eta+1$, the expected value of the measure $M$ given $M \geq m$ and $G = g$ is then
\begin{equation}
    \mathbb E[M | M\geq m \wedge G=g] = m+\frac{m-g}{\beta-2}.
\end{equation}
This is a decreasing function of $g$, which explains why, for larger values of $G$, $M$ is expected to take smaller values. This explains why $G$ and $M$ can become negatively correlated.

It is noteworthy that the correlation among top candidates can get as bad as $-1/2$ for $\beta \approx 4$, for very small discrepancy $\xi$, assuming a roughly equally selective filtering process of candidates $\omega \in \Omega$ based on $M$.

We illustrate the behavior of the correlation and the goal, depending on the quantile $\alpha$ or on the measure threshold $m$. Figure \ref{fig:G_given_M_uniform_power} features the expected value of $G$ given $M \geq m_\alpha$. As can be seen, optimizing over $M$ is initially productive to achieve larger values of $G$. However, over-optimization over $M$ then becomes counter-productive for $G$. In fact, we eventually obtain the same expected value as in the case with no optimization over $M$.

Figure \ref{fig:correlation_uniform_power} displays the results when the correlation is drawn as a function of the threshold $m$. As the threshold increases, the correlation decreases below 0, until eventually increasing towards 0. As $\alpha$ approaches 0, we see the same phenomenon. The correlation plummets below 0, and eventually goes to 0, as the expected value of the goal converges to the no-optimization case.

\subsection{Power law goal, power law discrepancy}

Now consider that the goal and the discrepancy follow power law distributions. We assume that the goal $G$ has a probability density function $p^G[g] = (\gamma-1) g^{-\gamma}$ for $g \geq 1$. We consider the same probability density function for $\xi$, that is, $p^\xi(x) = (\beta-1) \eta^{\beta-1} x^{-\beta}$ for $x \geq \eta$. 

Note that we now have $\mathbb E[G] = \frac{\gamma-1}{\gamma-2}$, $\mathbb E[G] = \frac{\gamma-1}{\gamma-3}$ and $Var(G) = \frac{\gamma-1}{(\gamma-2)^2(\gamma-3)}$. For the noise-to-signal ratio to equal $\varepsilon$, we need to have $\eta = \frac{\gamma-2}{\beta-2} \sqrt{\frac{(\beta-1) (\gamma-3)}{(\gamma-1)(\beta-3)}} \varepsilon$. As earlier, we assume $\beta > 3$ and $\gamma > 3$.

\begin{theorem}
\label{th:power_power}
Suppose that $G$ and $\xi$ follow a power law distribution with decay rates $g^{-\gamma}$ and $x^{-\beta}$. As $\alpha \rightarrow 0$, we have
\begin{equation}
\rho_\alpha \sim \left\{
\begin{array}{cc}
1, & \text{ if } \gamma- \beta < 0,  \\
\displaystyle - C \varepsilon^{-\frac{\gamma-1}{2}} \alpha^{\frac{\gamma-\beta}{2(\beta-1)}}, & \text{ if } \gamma-\beta \in (0,2), \\
\displaystyle - \frac{D \alpha^{\frac{1}{\beta-1}}}{\varepsilon}, & \text{ if } \gamma-\beta > 2,
\end{array}
\right.
\end{equation}
where $C$ and $D$ are positive constants that only depend on $\beta$ and $\gamma$ as follows.

\begin{align*}
C &= (\beta-2) \left( \frac{\gamma-1}{\gamma-2} \frac{\beta-1}{\beta-2}  - \frac{\gamma-1}{\gamma-3}  \right) \sqrt{\frac{(\gamma-3)(\beta-3)}{(\gamma-1)(\beta-1)}} >0, \\
D &= \frac{1+(\beta-1)(\gamma+\beta-3)}{\gamma-2} \sqrt{\frac{(\gamma-1)(\beta-3)}{(\gamma-3)(\beta-1)}} >0.
\end{align*}

\end{theorem}

\begin{proof}
See Appendix \ref{app:power}.
\end{proof}

\begin{figure}[!ht]
        \includegraphics[width=0.5\textwidth]{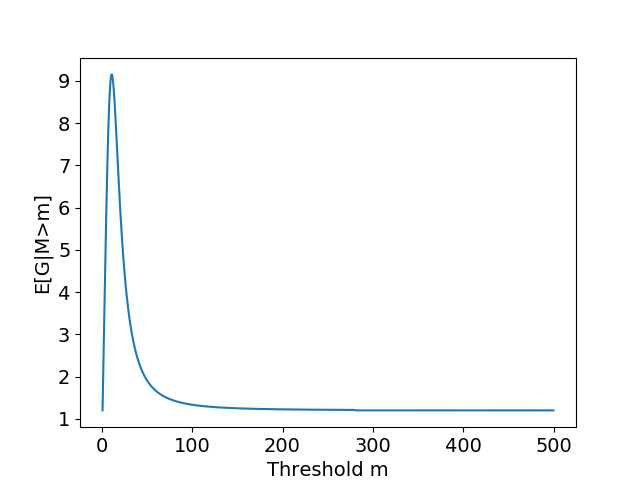} 
        \includegraphics[width=0.5\textwidth]{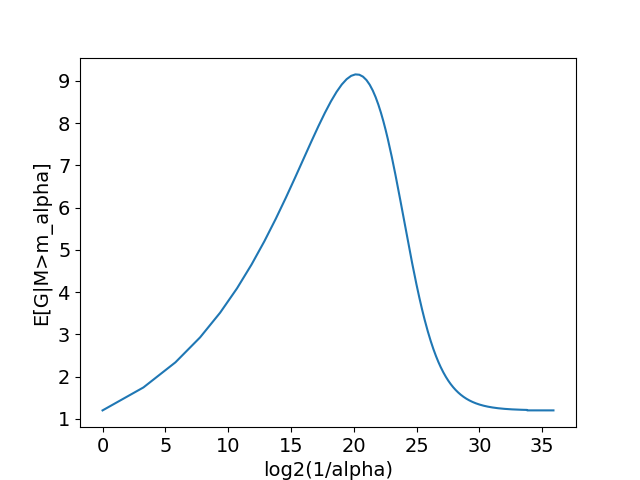} 
\caption{As $\alpha$ decreases, the measure threshold $m_\alpha$ increases. The expected goal $\mathbb E_\alpha[G]$ increases steadily. The figure depicts the expected goal as a function of the measure threshold (on the left) and as a function of $\log_2(1/\alpha)$ (on the right). The parameters were set at $\beta=3.5$, $\gamma=7$ and $\varepsilon = 1/256$.}
\label{fig:G_given_M_power_power}
\end{figure}

Theorem~\ref{th:power_power} shows that if the tail of the discrepancy $\xi$ is sufficiently thinner than that of the goal $G$ ($\gamma < \beta$), then the measure $M$ is an excellent measure of the goal $G$, especially as we optimize for $M$. However, if the discrepancy has a much thicker tail, then the measure $M$ becomes negatively correlated with the goal $G$. Optimizing for $M$ would then be greatly counter-productive to the goal $G$.

It is noteworthy that, as opposed to what we observed for previous cases, the Goodhart effect here only starts much later than when $\alpha \approx \varepsilon$. Indeed, in Figure \ref{fig:G_given_M_power_power}, we only obtain a strong Goodhart effect when $\log_2(\alpha) \approx 22$, even though $\log_2(\varepsilon) = 8$.

\subsection{Bounded goal, log-normal discrepancy}

In this subsection we investigate what happens when $G$ is uniform on $[0,1]$, and where the discrepancy $\xi$ is a log-normal distribution of parameters $(0,\eta)$, In other words, we assume 
\begin{equation*}
p^{\xi}(x) = \frac{1}{x \eta \sqrt{\tau}} \exp \left( - \frac{(\ln x)^2}{2 \eta^2} \right).
\end{equation*}
Note that, to obtain a noise-to-signal ratio of $\varepsilon$, the parameter $\eta$ needs to verify $e^{2 \eta^2}-e^{\eta^2}=\varepsilon^2$, which is equivalent to $\eta^2 = \ln \frac{1+\sqrt{1+4\varepsilon^2}}{2}$. 

Unfortunately, we have not been able to analyze theoretically Goodhart's law in this setting. But surprisingly, numerical computations suggest that the maximization of the measure $M$ yields a ``double descent'' for the goal. The expected value of the goal first increases, then decreases, and then increases again to reach better expected values of $G$ than after the first increase.

\begin{figure}[!ht]
        \includegraphics[width=0.5\textwidth]{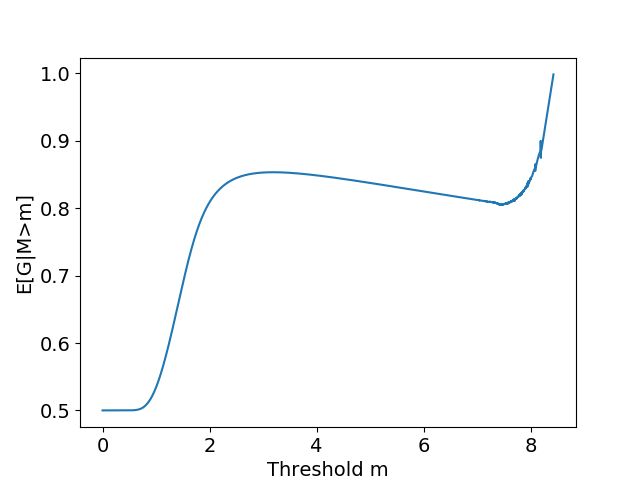} 
        \includegraphics[width=0.5\textwidth]{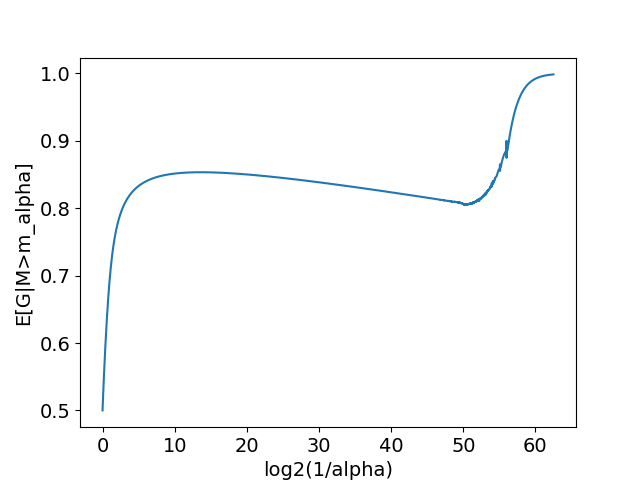} 
\caption{As $\alpha$ decreases, the measure threshold $m_\alpha$ increases. The figure illustrates this phenomenon as a function of the measure threhold (on the left) and as a function of $\log_2(1/\alpha)$ (on the right). The parameters were set at $\varepsilon = 1/4$.}
\label{fig:G_given_M_uniform_lognormal}
\end{figure}

There are important caveats though regarding both the reliability of numerical computations, as well as the robustness of this ``double descent'' for other (smaller) values of $\varepsilon$. More discussions are provided in Appendix \ref{app:lognormal}.

\subsection{Worst-case Goodhart's law}
\label{sec:worst_case}

How bad can Goodhart's law get? Is it possible for instance that, as we maximize the measure $M$, the goal $G$ becomes infinitely bad? In this section, by allowing a dependency between the goal $G$ and the discrepancy $\xi$, we answer in the affirmative.

\begin{theorem}
There are settings where maximizing the measure $M = G+\xi$ is eventually infinitely bad for the goal $G$, even if, given any value of the goal $G$, the discrepancy $\xi$ is centered and the noise-to-signal variance is $\varepsilon$. More precisely, even if $Var(G) = 1$, $\mathbb E[\xi|G=g] = 0$ and $Var(\xi|G=g) = \varepsilon^2$ for all $g$ in the support of $G$, we can have $\mathbb E_\alpha[G] \rightarrow - \infty$.
\end{theorem}

\begin{proof}[Sketch of proof]
The example we exhibit is one where $-G$ follows an exponential distribution, i.e. $p^G(g) = e^g$ for $g \leq 0$. We then consider power-law discrepancies whose power law decay satisfies $\beta_g = 4+\frac{1}{1-g}$ (we also add a mass probability at a negative point $-x_g$ to guarantee zero mean, and normalize the coefficient $\eta_g$ so that the discrepancy $\xi$ has variance $\varepsilon^2$ for all values of $g$). As a result, for more negative values of $g$, the discrepancy $\xi$ given $G=g$ has a thicker tail distribution. We then show that, in this setting, $\mathbb E_\alpha[G] \rightarrow - \infty$.

The full proof is available in Appendix \ref{app:worst_case}.\end{proof}

More generally, this result suggests that, in general, if the discrepancy $\xi$ has a very long tail for some value of $G$, over-optimizing $M$ may imply that $G$ will tend to take values for which the tail of $\xi$ is the thickest. If such tails are thickest for very undesirable values of $G$, the over-optimization of $M$ may thus favor very undesirable outcomes.

\section{Why robust alignment may be very hard}
\label{sec:alignment}

\subsection{Illustrative example}
\label{sec:illustrative}




Let us first start with an illustrative example with a simple interaction model between an algorithm and a user. We assume that the algorithm recommends a content $x_t$ at each time step $t$, where $x_t \in \mathbb R$ is simply a real number. The user that consumes the content, and retrieves some satisfaction that is larger if $x_t$ is close to the user's intrinsic preference $\theta \in \mathbb R$. Moreover, the user is asked whether they would prefer the content to take a larger or a smaller value of $x_t$. However, here, we assume that the user's declared preference is biased by the user's addiction to the contents they have consumed in the past. This feature is key, as it corresponds to a feedback loop, which we argue to increase the risk of Goodhart's law. 

More precisely, we assume that their instantaneous preference at time $t$ is some combination of their intrinsic preference $\theta$, with a factor $\alpha$ representing some form of addiction coefficient to $\theta$, altered by the accumulated exposition they had to previous content $x_j, j\leq t$ and is given by 
$$\theta_t \triangleq \frac{\alpha \theta + \sum_{j=1}^t x_j}{\alpha + t}.$$
If $x_t \leq \theta_t$, then the user answers that they wishe $x_t$'s to take larger values, i.e. $y_t \triangleq +1$. Otherwise, the user replies $y_t \triangleq -1$. The algorithm then chooses a next content $x_{t+1}$ based on past proposals and user's responses. In fact, we limit ourselves to algorithms that output $x_{t+1}$ by computing
$$x_{t+1} \triangleq x_t + \frac{\omega y_t}{t},$$
and that begin with $x_0 = 0$. We then consider algorithms that maximize the user's estimated satisfaction within $T=100$ rounds,
$$H(\omega, \theta, \alpha) \triangleq \sum_{t=1}^T \frac{1}{\delta + (\theta_t-x_t)^2},$$
with $\delta \triangleq 10^{-5}$. 

We assume that the measure used to select which algorithm to deploy has a correct estimation of $\theta = 0.2$. However, let us assume that this measure gets the estimation of the addictiveness $\alpha$ wrong. The measure uses $\alpha_M = 50$,  while the goal uses $\alpha_G = 5$, which means that the measure under-estimates how much the user will be addicted to the content they consume. In other words, we have $M(\omega) = H(\omega, \theta, \alpha_M)$, while $G(\omega) = H(\omega, \theta, \alpha_G)$. We then drew $\omega$ from a log-normal distribution of parameters $\mu = 0$ and $\sigma =3$ to cover a wide range of values of $\omega$. The correlation between goal and measure was $0.8$, and is depicted in Figure \ref{fig:Experiment_correlation}.

\begin{figure}[!ht]
        \includegraphics[width=0.5\textwidth]{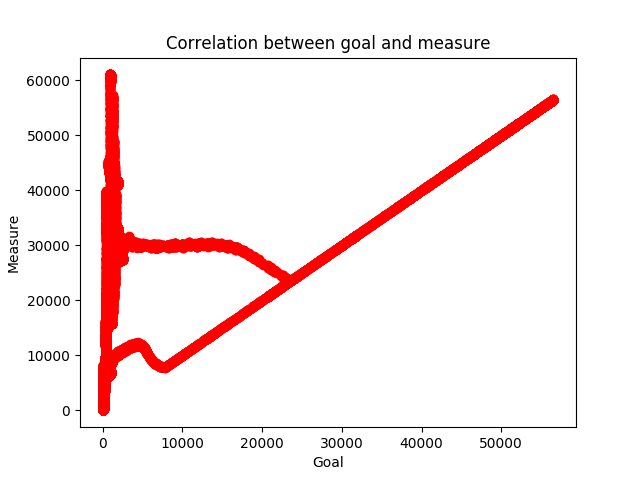} 
        \includegraphics[width=0.5\textwidth]{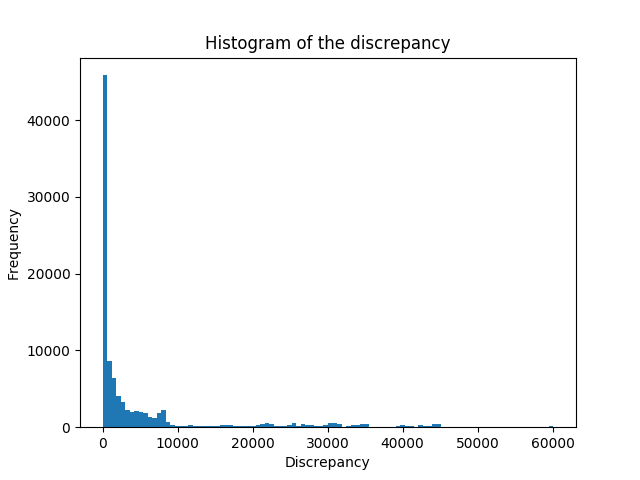} 
\caption{On the left, the correlation between goal and measure is depicted for our simple example of an algorithm interacting with a user prone to an under-estimated addiction. The figure on the right displays the distribution of the discrepancy between goal and measure.}
\label{fig:Experiment_correlation}
\end{figure}

After drawing $100,000$ values of $\omega$, the top value of the goal we obtained is $G(\omega^*_G) = 56539$, using $\omega^*_G = 0.31$ (and we have $M(\omega^*_G) = 56539$). However, the top value of the measure was $M(\omega^*_M) = 61099$ with $\omega^*_M = 0.19$, which is associated with a remarkably low value of the goal, as $G(\omega^*_M) = 986$. It is noteworthy that this instance of Goodhart's law seems to be associated with a heavy-tailed distribution of the discrepancy between the measure and the goal, as depicted by Figure \ref{fig:Experiment_correlation}.

While this simple example should not be regarded as a compelling evidence that Goodhart's law applies in practice, it is arguably suggestive of the fact that the risk of dangerous over-optimization cannot be excluded, especially in the presence of feedback loops.

\subsection{On the hardness of robust alignment}

It is common to test algorithms $\omega \in \Omega$ with a fixed (series of) test $T$. In other words, the algorithm $\omega$ is designed to optimize some given objective function $M_T(\omega)$ before deployment. This measure $M_T(\omega)$ may be some performance measure on historical data. However, such a measure is likely to be very different from the actual performance $G(\omega)$ of the algorithm $\omega$ once deployed.

One critical difference between $M_T$ and $G$ is that, once deployed, the algorithm $\omega$ may modify its environment, and thus the distribution $\mathcal D(\omega)$ of future data. It was, after all, the point of deploying the algorithm. Algorithms are deployed so as to modify (and hopefully improve) their environment. 

But as a result, the performance of the algorithm once deployed is actually dependent on the algorithm in two ways. First, the algorithm achieves some accuracy on its task. But second, and sometimes more importantly, it creates a {\it distributional shift} that changes the set of tasks it will be tested on. As an example, a recommender system might change users' preferences by creating addiction in them, which will then facilitate click-through rate maximization. By doing so, the algorithm actually modifies $T$ (the users on which it is tested) and creates the aforementioned distributional shift.

We can formalize this by considering that the performance of the algorithm is measured in terms of some loss function $\mathcal L$ that depends on data. So, typically, we would have $M_T(\omega) = \mathbb E_{x \leftarrow T} [-\mathcal L(x,\omega)]$, while $G(\omega) = \mathbb E_{x \leftarrow \mathcal D(\omega)} [-\mathcal L(x,\omega)]$. 

In the context of Goodhart's law, this is critical to note, as we now have a feedback loop between the outputs of the algorithm and its future inputs given by its environment. Such feedback loops might then lead to power law distributions (see \cite{dorogovtsev2000,almgren2005,brynjolfsson2010} among others), which would open the door for some strong Goodhart's law. If this holds, the over-optimization of the measure might eventually be (extremely) counter-productive to the goal.

Note that for this to hold, the algorithm $\omega$ need not be learning from the data $x$ drawn from $\mathcal D(\omega)$. Our basic assumption here is merely that the algorithm $\omega$ modifies its environment (thus modifies $T$) by inducing a distributional shift. Having said this, the risk of emergence of some power law discrepancy seems larger if the algorithm is also learning from the future inputs given by the environment, especially if such inputs are {\it adversarial} (see \cite{keller2019} for an example of attacks on algorithms). 

Another aggravating factor could be the misalignment between the test performances $M_T(\omega)$ of algorithms $\omega$, and their impacts $G(\omega)$ measured in terms of what we actually value. Typically, this goal $G(\omega)$ might be what we would really want to maximize as a society, like for instance an adequate aggregation of humans' preferences which we did not perfectly model. Unfortunately, we do not yet have reliable algorithms to compute the function $G$ (and we may never have such an algorithm!). As a result, we need to content ourselves with proxies $M_T$ that seem correlated enough with $G$. But then, the over-optimization of proxies may be extremely harmful to our true goal $G$.

In the light of our results, it seems worth discussing the potential tail distributions of the goal $G$ and of the discrepancy $\xi = M_T-G$. What may happen in practice is that, while $M_T(\omega)$ does not feature a long tail distribution for large positive values, both $G(\omega)$ and $\xi_T(\omega) = M_T(\omega)-G(\omega)$ might. Interestingly, the positive tail distribution of the discrepancy $\xi_T$ would then be essentially the negative tail distribution of the goal $G(\omega)$. Thus, the flavor of Goodhart's law that would apply may actually not depend that much on our measure $M_T$; it might rather depend on the true goal $G$ (and on our choice of the distribution $\tilde \Omega$ of candidate algorithms $\omega$ to be deployed). In particular, assuming that $G$ has long tails, as suggested by Theorem \ref{th:power_power}, the harmfulness of over-optimizing $M_T$ might actually mostly depend on which side (positive or negative) of the distribution of $G$ has a thicker tail.

Let us stress that this discussion is somewhat speculative though, since the goal $G$ and the discrepancy $\xi_T$ here are probably not independent. Thus, our above theorems do not apply. Nevertheless, while informal, the above remarks warn us about the risks of counter-productive side effects, once optimized algorithms are deployed in complex environments. In particular, this suggests that the {\it alignment problem}, i.e. the problem of designing the objective functions to be maximized, may be harder than expected. In particular, because of the risk of a strong Goodhart's law and in the light of our results, especially in complex environments with feedback loops, a mere correlation between a given measure and a goal should not count as a strong argument for the safety of maximizing that particular measure to achieve the goal. Such an alignment would likely not be a {\it robust} alignment.

\section{Remarks and future works}
\label{sec:remarks}



This work raises numerous intriguing questions regarding the potential pitfalls of optimization, and the need of algorithms that are resilient to counter-productive over-optimization. In the following, we list numerous research directions for future work and provide pointers to where initial efforts are already undertaken. 

{\bf Understanding the log-normal discrepancy case}. Numerical simulations suggest that for uniform goal and log-normal discrepancy, Goodhart's law features a double descent phenomenon. Whether or not we can mathematically prove this numerical  finding remains an open question. If so, can we predict the point where the second descent starts?

{\bf Application to overfitting}. Overfitting seems to be a case of Goodhart's law. However, it is not yet clear how our analysis could be best used to predict overfitting. In particular, it is not clear what probability distribution over parameters of a machine learning algorithm is most fit to describe the likely values of parameters after sufficient optimization. Can the risk of counter-productive overfitting be predicted before spending huge amounts of power to optimize neural network parameters? Can our approach actually explain double descent? Can it be used to explain the success of interpolators?

{\bf Tail-skewed goal}. We have suggested that an unknown (or hard-to-estimate) goal $G$ could be particularly hard to optimize if its negative tail distribution is heavier than its positive tail distribution. Could this assumption of skewness of tail distribution be discarded? Or is it possible to derive it from reasonable assumptions? Are human values likely to have skewed tail distributions? If so, what can be done? How can we avoid a strong Goodhart effect despite a tail-skewed goal?

{\bf More complete analysis of Goodhart's law}. This paper only presented five settings of goal-discrepancy tail distributions. It also only mostly analyzed asymptotic behaviors as the measure is being maximized. Is there a more insightful theory to predict Goodhart effect, including in non-asymptotical regimes?

{\bf Biased training set}. In classical learning theory, it has become common to assume that the training set is an unbiased sample of the out-of-sample data distribution. But in practice, there is often a sample bias, simply because easier-to-collect data are more likely to be in the training set. How dangerous is it to learn from biased datasets? Further analysis of Goodhart's law in this setting could shed new light on this question, critical to algorithmic bias and fairness.

{\bf Robustness to poisoning.} In the light of our results, one could ask whether bad actors could introduce well-crafted data in order to increase the heavy-tailness of the discrepancy. For instance, a malicious attacker might know that the algorithm computes the maximum a posteriori as a proxy for the posterior distribution, and inject poisoned data to induce a stronger Goodhart effect. This question could serve as a good intersection between the growing interest in the alignment problem and the recent advances in robust statistics~\cite{diakonikolas2019robust}. More concrete examples to motivate this is the poisoning of social media with misinformation, such as medical advice against vaccines, or recommendations for false cures, as the proxy being maximized is often ``engagement'', social media algorithms tend to be easily fooled by this type of contempt and amplify it further\cite{faucon2021recommendation}.

{\bf Heavy tailed statistics.} Closely related to the last question is the general effort to better understand the behaviour of learning algorithms under heavy tailed situations~\cite{hsu2016loss, lecue2017robust}. A very recent work~\cite{zhang2019adam} is, e.g. suggesting that the answer to the poor performance of stochastic gradient descent compared to ADAM in some situations might lie in the heavy-tailed distribution of the noise.

{\bf Diagnosing strong Goodhart's law}. Our work suggests to study the tail distribution of $\xi = M-G$. But what is the best way to infer the tail distribution of $\xi$ from a sample? Studying the tail distribution of $\xi$ is challenging in itself, but another challenge is to tell whether $\xi$ is or is not independent from $G$.  Further challenges to explore is whether or not one can draw sufficiently many samples to estimate the aforementioned tail distribution of $\xi$. Another issue will naturally arise in this direction as we often only have a biased sample of $\xi$'s.  To what extent can we still estimate the risk of strong Goodhart's law given these challenges remain an open question, of which answers are of highly practical consequences?

{\bf Early stopping}. Our work suggests that, at least in some cases and assuming a risk of Goodhart's law, it would be safer to optimize up to (roughly) the $\varepsilon$-quantile. But is this observation robust? Can we further analyze other properties of the measure $M$, the goal $G$ and the distribution $\tilde \Omega$ to better identify the point after which optimization may become counter-productive? Moreover, in practice, we do not know the goal $G$, which prevents us from estimating the correlation between the measure and the goal. How can we pragmatically determine a robust early stopping criterion?

{\bf Combining measures}. In the same vein as the previous conjecture, suppose we have $n$ measures $M_1, ... M_n$, such that the discrepancies $\xi_k = M_k-G$ are independent. Then averaging $M= \frac{1}{n} \sum M_k$ should allow to reduce the errors to be roughly $\varepsilon = \frac{1}{\sqrt{n}} \sqrt{\sum \varepsilon_k^2}$. But this will actually likely not be the most robust aggregation, as it will be as heavy-tailed as the heaviest-tailed discrepancy. In fact, robust statistics provides tools to construct sub-exponential estimators from heavy-tailed distribution \cite{lecue2019}. Exploring this direction yields other interesting questions, as the $M_k$'s might not be independent, have different variances, and poses other challenges when inferring their dependency and their variance from raw data. This direction however seems to be a practical way to guarantee resilience against strong Goodhart's law, as the averaged measure is more complex proxy to the goal than every single measure taken alone. Actually, this is what is done in practice, when a student is judged by different examinations, when a scientific result is evaluated by several (ideally independent) reviewers, or when policies have to be approved by several policy making bodies. The rationale is that ``Goodharting'' one evaluator might be easy, but Goodharting\footnote{The term Goodharting emerged informally among some AI researchers and means, roughly, acting in a way that hacks the proxy measure, without necessarily meeting the targets, desired by the main goal.} all of them at once becomes harder.

{\bf Random tests}. Suppose we have $n$ measures $M_1, ... M_n$, such that the errors $\xi_k = M_k-G$ are independent. Then enforcing some uncertainty upon which measure $M_k$ (and assuming the optimizer maximizes its expected score) will be used may have an effect similar to averaging the measures $M_k$. If each measure $M_k$ is costly to compute, this could allow the same efficiency with significantly less computational power. Can this trade-off be quantified, and improved upon? Can the random sampling be made robust to strong Goodhart's law, despite the heavy-tailness of some measures?

{\bf Quantifying the need to improve measures}. In practice, the quality of the measure $M_Z$ depends on the investment $Z$ in data collection, data cleaning, measure design and troubleshooting. This measure $M_Z$ will likely be more and more correlated with the goal for larger values of $Z$, but companies, institutions and governments may be reluctant to invest too large values. Is there a principled way of choosing the investment $Z$ to optimize with provable guarantees? Can we design an efficient algorithm to adaptively choose the investment $Z$ to maximize in a Goodhart-resilient manner?

\section{Conclusion}
\label{sec:conclusion}

This paper showed that Goodhart's law strongly depends on the tail distributions of the goal and the discrepancy between the measure and the goal. We introduced and distinguished between the {\it weak} and the {\it strong} Goodhart's law. Longer-tailed discrepancies create stronger Goodhart effect, which, in the worst-case, can even be infinitely bad for our true goal. Our results also suggest that the alignment problem is likely to be particularly difficult in complex environments that involve feedback loops. They warrant more investigation into the design of measures that are resilient to Goodhart's law.

\bibliographystyle{alpha}
\bibliography{references}

\begin{appendix}
\section{General facts}
\label{app:general_facts}

\begin{lemma}
Let $p$ be the joint probability density function of $G$ and $\xi$. Then $p$ satisfies the following property:
\begin{equation*}
p_\alpha \left[ g,x \right] = \frac{p^\xi(x) p^G(g)}{\alpha} {\bf 1}_{g+x \geq m_\alpha}.
\end{equation*}
\end{lemma}

\begin{proof}
By definition of $p$ and using $M=G + \xi $, we have
\begin{equation*}
p_\alpha\left[ g,x\right] = \frac{p\left[ G=g,\xi=x \wedge G+\xi \geq m_\alpha \right]}{\mathbb P[G+\xi \geq m_\alpha]}.
\end{equation*}
We conclude by noting that $p\left[ G=g,\xi=x \wedge G+\xi \geq m_\alpha \right]$ is equal to zero if $g+x < \alpha$, and equal to $p[G=g,\xi=x] = p^G(g) p^\xi(x)$ otherwise (using independence of $G$ and $\xi$), and that $\mathbb P[G+\xi \geq m_\alpha] = \mathbb P[M \geq m_\alpha] = \alpha$, by definition of $m_\alpha$.
\end{proof}

Using the lemma, we obtain the following simplifying equalities
\begin{align*}
\alpha &= \int_{g=-\infty}^{+\infty} \left( \int_{m_\alpha-g}^\infty p^\xi \right) p^G[g] dg \\
\mathbb E_\alpha[G] &= \frac{1}{\alpha} \int_{g=-\infty}^{+\infty} \left(\int_{m_\alpha - g}^\infty p^\xi \right) g p^G[g] dg, \\
\mathbb E_\alpha[\xi] &= \frac{1}{\alpha} \int_{g=-\infty}^{+\infty} \left( \int_{x=m_\alpha-g}^\infty x p^\xi(x) dx \right) p^G[g] dg, \\
\mathbb E_\alpha[G^2] &= \frac{1}{\alpha} \int_{g=-\infty}^{+\infty} \left(\int_{m_\alpha - g}^\infty p^\xi \right) g^2 p^G[g] dg, \\
\mathbb E_\alpha[\xi^2] &= \frac{1}{\alpha} \int_{g=-\infty}^{+\infty} \left( \int_{x=m_\alpha-g}^\infty x^2 p^\xi(x) dx \right) p^G[g] dg, \\
\mathbb E_\alpha[G\xi] &= \frac{1}{\alpha} \int_{g=-\infty}^{+\infty} \left( \int_{x=m_\alpha-g}^\infty x p^\xi(x) dx \right) g p^G[g] dg. 
\end{align*}
Also, note that we have the following equality
\begin{align*}
Var_\alpha(G) &= \mathbb E_\alpha[G^2] - \mathbb E_\alpha[G]^2, \\
Var_\alpha(\xi) &= \mathbb E_\alpha[\xi^2] - \mathbb E_\alpha[\xi]^2, \\
Cov_\alpha(G,\xi) &= \mathbb E_\alpha[G\xi] - \mathbb E_\alpha[G] E_\alpha[\xi],
\end{align*}
Finally, this can all be combined to obtain the correlation among top $\alpha$ values of $M$:
\begin{equation*}
\rho_\alpha = \frac{Var_\alpha(G) + Cov_\alpha(G,\xi)}{\sqrt{Var_\alpha(G) \left( Var_\alpha(G) + Var_\alpha(\xi) + 2 Cov_\alpha(G,\xi) \right)}}.
\end{equation*}

\section{Bounded goal, exponential noise}
\label{app:uniform_exponential}

\begin{lemma}
If $\alpha \leq \varepsilon \sqrt{12} (1-e^{-1/\varepsilon\sqrt{12}})$, then $m_\alpha \geq 1$.
\end{lemma}

\begin{proof}
Starting from $M=G+\xi$, we have the following 
\begin{align*}
\mathbb P[M \geq 1] &= \mathbb P[G+\xi \geq 1] = \mathbb E_{G} \left[ \mathbb P[\xi \geq 1-G | G] \right] \\
&= \int_{g=0}^{1} \mathbb P[\xi \geq 1-g] dg = \int_{g=0}^1 e^{-\lambda(1-g)} dg \\
&= e^{-\lambda} \left[ \frac{e^{\lambda g}}{\lambda} \right]_{g=0}^1 = \frac{1}{\lambda} (1-e^{-\lambda}) \geq \alpha, 
\end{align*}
where, in the last line, we used $\lambda = 1 /\varepsilon \sqrt{12}$. This allows to conclude that $m_\alpha \geq 1$.
\end{proof}

\begin{lemma}
 $\lambda \alpha = e^{-\lambda m_\alpha} (e^{\lambda h_\alpha} - 1) + \lambda(1-h_\alpha)$, where $h_\alpha = \min(m_\alpha,1)$.
\end{lemma}

\begin{proof}
Note that if $m_\alpha \geq g$, then $\int_{m_\alpha-g}^\infty p^\xi = e^{-\lambda (m_\alpha-g)}$, and it equals 1 otherwise. Using the equations of Section \ref{app:general_facts}, we have
\begin{align*}
\alpha &= \mathbb P[M \geq m_\alpha] 
= \int_0^{h_\alpha} e^{-\lambda(m_\alpha-g)}dg + \int_{h_\alpha}^1 dg \\
&= e^{-\lambda m_\alpha} \left[ \frac{e^{\lambda g}}{\lambda} \right]_0^{h_\alpha} + 1-h_\alpha \\
&= \frac{e^{-\lambda m_\alpha}}{\lambda} (e^{\lambda h_\alpha} -1) + 1-h_\alpha.
\end{align*}
Multiplying by $\lambda$ on both sides yields the lemma.
\end{proof}

\begin{lemma}
$\mathbb E_\alpha[G] = \frac{1}{2 \lambda}\frac{\lambda^2 (1-h_\alpha^2) e^{\lambda m_\alpha} + 2 \lambda h_\alpha e^{\lambda h_\alpha} - 2 e^{\lambda h_\alpha} + 2}{\lambda (1 - h_\alpha) e^{\lambda m_\alpha} + e^{\lambda h_\alpha}-1}$.
\end{lemma}

\begin{proof}
Using the equations of Section \ref{app:general_facts} and integration by parts, we have the following computation:
\begin{align*}
\alpha \lambda^2 e^{\lambda m_\alpha} &\mathbb E_\alpha[G] = \lambda^2  e^{\lambda m_\alpha} \int_0^{h_\alpha} g e^{-\lambda(m_\alpha-g)} dg + \lambda^2 e^{\lambda m_\alpha} \int_{h_\alpha}^1 g dg \\
&= \lambda^2 \int_0^{h_\alpha} g e^{\lambda g} dg + \frac{\lambda^2 e^{\lambda m_\alpha}}{2} \left[ g^2 \right]_{h_\alpha}^1 \\
&= \lambda^2 \left( \left[ g \frac{e^{\lambda g}}{\lambda} \right]_0^{h_\alpha} - \int_0^{h_\alpha} \frac{e^{\lambda g}}{\lambda} dg \right) + \frac{\lambda^2 (1-h_\alpha^2)}{2} e^{\lambda m_\alpha} \\
&= \lambda h_\alpha e^{\lambda h_\alpha} - \left[ e^{\lambda g} \right]_0^{h_\alpha} + \frac{\lambda^2 (1-h_\alpha^2)}{2} e^{\lambda m_\alpha} \\
&= \lambda h_\alpha e^{\lambda h_\alpha} - e^{\lambda h_\alpha} + 1 + \frac{\lambda^2 (1-h_\alpha^2)}{2} e^{\lambda m_\alpha} \\
&= \frac{\lambda \alpha}{2} \frac{\lambda^2 (1-h_\alpha^2) e^{\lambda m_\alpha} + 2 \lambda h_\alpha e^{\lambda h_\alpha} - 2 e^{\lambda h_\alpha} + 2}{\lambda (1-h_\alpha) + e^{\lambda (h_\alpha-m_\alpha)}-e^{-\lambda m_\alpha}},
\end{align*}
where, in the last line, we used the previous lemma.
\end{proof}

This lemma allows us to plot $\mathbb E_\alpha[G]$ as a function of $m_\alpha$. This is done in Figure \ref{fig:G_given_M_uniform_exp}. Putting everything together, we can finally prove Theorem \ref{th:uniform_exp}.

\begin{proof}[Proof of Theorem \ref{th:uniform_exp}]
Consider $\alpha \leq \varepsilon \sqrt{12} (1-e^{-1/\varepsilon\sqrt{12}})$. Then, we have $\mathbb E_\alpha[G M] = \mathbb E_\alpha^G \left[ G \mathbb E_\alpha [M|G] \right] = \mathbb E_\alpha^G \left[G (m_\alpha+1/\lambda) \right] = \mathbb E_\alpha^G \left[G \right] (m_\alpha+1/\lambda) = \mathbb E_\alpha [G] \mathbb E_\alpha[M]$, which corresponds to saying $Cov_\alpha(G,M) = 0$. This implies $\rho_\alpha = 0$.

For the expectation of $G$, the above lemma does most of the work. We can conclude by noting that, when $m_\alpha \geq 1$, we have $h_\alpha = 1$, which implies $\mathbb E_\alpha [G] = \frac{1-\frac{1}{\lambda} + \frac{e^{-\lambda}}{\lambda}}{1-e^{-\lambda}} = 1-\frac{1}{\lambda} + o(1/\lambda)$ as $\lambda \rightarrow \infty$.
\end{proof}

\section{Normal distribution}
\label{app:normal}

\subsection{Useful lemmas}

\begin{lemma}
$\int_m^\infty e^{-\frac{t^2}{2\sigma^2}}dt = \frac{\sigma^2}{m} e^{-\frac{m^2}{2\sigma^2}} \left( 1- \frac{\sigma^2}{m^2} + \frac{3 \sigma^4}{m^4} + O\left(\frac{1}{m^6} \right) \right)$.
\end{lemma}

\begin{proof}
It is well known that, as $x \rightarrow \infty$, we have the following integrated Taylor expansion: 
\begin{equation*}
\int_x^\infty e^{-u^2} du = \frac{e^{-x^2}}{2x} \sum_{n=0}^{N-1} (-1)^n \frac{(2n - 1)!!}{(2x^2)^n} + O\left( x^{-2N-1} e^{-x^2} \right).
\end{equation*}\footnote{The notation $!!$ denotes validity on odd values of $k \leq n$ only.}
Taking the first terms of this expansion, and replacing $u$ by $t/\sigma\sqrt{2}$ then yields the lemma.
\end{proof}

\begin{lemma}
$\int_m^\infty t^2 e^{-\frac{t^2}{2\sigma^2}}dt = m \sigma^2 e^{-\frac{m^2}{2\sigma^2}} \left( 1+ \frac{\sigma^2}{m^2} - \frac{ \sigma^4}{m^4} + O\left(\frac{1}{m^6} \right) \right)$.
\end{lemma}

\begin{proof}
We use integration by parts:
\begin{align*}
\int_m^\infty t \cdot t e^{-\frac{t^2}{2\sigma^2}}dt &= -\sigma^2 \left[ t e^{-\frac{t^2}{2\sigma^2}} \right]_m^\infty + \sigma^2 \int_m^\infty e^{-\frac{t^2}{2\sigma^2}} dt \\
&= m \sigma^2 e^{-\frac{m^2}{2\sigma^2}} + \frac{\sigma^4}{m} e^{-\frac{m^2}{2\sigma^2}} \left( 1- \frac{\sigma^2}{m^2} + O(m^{-4})\right),
\end{align*}
which yields the lemma.
\end{proof}

\begin{lemma}
\label{lemma:main_normal}
Assume $\sigma^2 = \delta^2 + \kappa^2$, with $\delta, \kappa > 0$. Then, as $m \rightarrow \infty$,
\begin{align*}
\iint_{u+v \geq m} u^2 & e^{-\frac{u^2}{2\delta^2}-\frac{v^2}{2\kappa^2}} dv du \nonumber \\
&= \frac{\delta^5 \kappa m \sqrt{\tau} e^{-\frac{m^2}{2 \sigma^2}}}{\sigma^3} \left( 1 + \frac{\sigma^4}{\delta^2 m^2} - \frac{\sigma^6}{\delta^2 m^4} + O(m^{-6}) \right)
\end{align*}
\end{lemma}

\begin{proof}
Let $\gamma = \frac{\kappa+\sigma}{2\kappa}$. Since $\sigma>\kappa$, we know that $1< \gamma <\sigma/\kappa$. These two inequalities will be useful later on.

Let us now decompose the double integral, by first integrating over $v$, and then over $u$. Then, we decompose the integration over $v$ by introducing the midpoint $m/\gamma$. This yields
\begin{align*}
&\iint_{u+v \geq m} u^2  e^{-\frac{u^2}{2\delta^2}-\frac{v^2}{2\kappa^2}} dv du
= \int_{v=-\infty}^{+\infty} \int_{u=m-v}^\infty u^2  e^{-\frac{u^2}{2\delta^2}} du e^{-\frac{v^2}{2\kappa^2}} dv\\
&= \int_{-\infty}^{m/\gamma} \int_{m-v}^\infty u^2  e^{-\frac{u^2}{2\delta^2}} du e^{-\frac{v^2}{2\kappa^2}} dv + \int_{m/\gamma}^{+\infty} \int_{m-v}^\infty u^2  e^{-\frac{u^2}{2\delta^2}} du e^{-\frac{v^2}{2\kappa^2}} dv.
\end{align*}
Now, the second term is going to be decaying much faster to 0 as $m\rightarrow \infty$, because $\gamma < \sigma$. To show this formally, note that 
\begin{align*}
\Big| \int_{m-v}^\infty & u^2 e^{-\frac{u^2}{2\delta^2}} du \Big| \leq \int_{-\infty}^{+\infty} u \cdot u e^{-\frac{u^2}{2\delta^2}} du \\
&= -\delta^2 \left[ u e^{-\frac{u^2}{2\delta^2}} \right]_{-\infty}^{+\infty} + \delta^2 \int_{-\infty}^{+\infty} e^{-\frac{x^2}{2\delta^2}} dx = \delta^3 \sqrt{\tau}.
\end{align*}
As a result, the second term can be upper-bounded as follows:
\begin{align*}
\Big| \int_{m/\gamma }^{+\infty} &\int_{m-v}^\infty u^2  e^{-\frac{u^2}{2\delta^2}} du e^{-\frac{v^2}{2\kappa^2}} dv \Big| \leq \delta^3 \sqrt{\tau} \int_{m/\gamma}^\infty e^{-\frac{v^2}{2 \kappa^2}} dv \\ 
&= O \left( \frac{1}{m} e^{-\frac{m^2}{2 \kappa^2 \gamma^2}} \right) = o\left( m^{-5} e^{-\frac{m^2}{2\sigma^2}} \right),
\end{align*}
since $\kappa \gamma < \sigma$.

Let us now focus on the first term. Importantly, the second integral goes from $m-v$ to infinity, where $v$ is now at most $m / \gamma$. Thus, the lower bound of the integral is at least $(1-\gamma^{-1})m$. Since $\gamma>1$, this means that the lower bound goes to infinity as $m$ goes to infinity. This allows us to use the asymptotic approximation of the previous lemma. Namely, assuming $v < m/\gamma$, we know that there exists a function $R$ such that
\begin{equation*}
\int_{m-v}^\infty u^2 e^{-\frac{u^2}{2\delta^2}} dx = (m-v) \delta^2 e^{-\frac{(m-v)^2}{2 \delta^2}} \left( 1+\frac{\delta^2}{(m-v)^2} - \frac{\delta^4}{(m-v)^4}+ \frac{R(m-v)}{(m-v)^6} \right),
\end{equation*}
such that $|R(m-v)| \leq C$ for some constant $C$ determined by considering, say $m \geq 1$ which implies $m-v \geq (1-\gamma^{-1})m \geq \Omega(1)$. 

Now notice that we have the equality
\begin{equation*}
\frac{(m-v)^2}{2\delta^2} + \frac{v^2}{2\kappa^2} = \frac{\sigma^2}{2 \kappa^2 \delta^2} \left( v-\frac{\kappa^2 m}{\sigma^2} \right)^2 + \frac{m^2}{\delta^2},
\end{equation*}
which can be verified by simply developing the squares. 

For $k \in \mathbb Z$, let us denote
\begin{equation*}
I_k = \int_{-\infty}^{m/\gamma} (m-v)^{-k} \exp \left(-\frac{\sigma^2}{2 \delta^2 \kappa^2}  \left(v-\frac{\kappa^2 m}{\sigma^2}\right)^2 \right) dv.
\end{equation*}
Then, using the three equations above, we can write the first term of our integrals as
\begin{equation*}
\int_{-\infty}^{m/\gamma} \int_{m-v}^\infty u^2  e^{-\frac{u^2}{2\delta^2}} du e^{-\frac{v^2}{2\kappa^2}} dv = \delta^2 e^{-\frac{m^2}{2\sigma^2}} (I_{-1} + \delta^2 I_1 - \delta^4 I_3 + I_R),
\end{equation*}
where 
\begin{equation*}
I_R = \int_{-\infty}^{m/\gamma} \frac{R(m-g)}{(m-v)^5} \exp \left(-\frac{\sigma^2}{2 \delta^2 \kappa^2}  \left(v-\frac{\kappa^2 m}{\sigma^2}\right)^2 \right) dv.
\end{equation*}
Interestingly, we observe that $|I_R| \leq C |I_5|$.

To compute an approximation of $I_k$, we make the change of variable $h=v-\frac{\kappa^2 m}{\sigma^2}$, and note that $m-v = \frac{\delta^2 m}{\sigma^2} - h$. The upper bound of the integral then becomes $\beta m$, where $\beta = \frac{1}{\gamma} - \frac{\kappa^2}{\sigma^2}$. Note that, since $1< \gamma < \sigma/\kappa$, we know that $1/\gamma > \kappa^2/\sigma^2$, which implies $\beta >0$.

We also separate the integral in two parts, by introducing a middle point $\sqrt{m}$. This allows us to write $I_k = J_k^1 + J_k^2$, where
\begin{align*}
J_k^1 &= \int_{-\infty}^{\sqrt{m}} \left(\frac{\delta^2}{\sigma^2} m-h\right)^{-k} e^{-\frac{\sigma^2 h^2}{2 \delta^2 \kappa^2}} dh, \\
J_k^2 &= \int_{\sqrt{m}}^{\beta m} \left(\frac{\delta^2}{\sigma^2} m-h\right)^{-k} e^{-\frac{\sigma^2 h^2}{2 \delta^2 \kappa^2}} dh.
\end{align*}
Let us first show that $J_k^2$ is exponentially negligible in $m$. To do so, note that for $h \leq \beta m$, we have $\frac{\delta^2}{\sigma^2} m-h \geq \left( \frac{\delta^2}{\sigma^2}+\frac{\kappa^2}{\sigma^{2}} - \frac{1}{\gamma}\right)m = (1-\gamma^{-1})m = \Omega(m)$. As a result, we have
\begin{equation*}
|J_k^2| \leq (1-\gamma^{-1})^{-k} m^{-k} \int_{\sqrt{m}}^\infty e^{-\frac{\sigma^2 h^2}{2 \delta^2 \kappa^2}} dh = O(m^{-k-\frac{1}{2}} e^{-\frac{\sigma^2 m}{2\delta^2 \kappa^2}}). 
\end{equation*}
In particular $J^2_k = O(m^{-5})$.

Let us now compute $J_k^1$ for $k=-1$. The trick is to note that it is the full integral over real numbers, minus the part between $\sqrt{m}$ and $+\infty$, which is exponentially negligible. And to also develop the term to obtain different integrals. More precisely, we have
\begin{align*}
J_{-1}^1 &= \frac{\delta^2 m}{\sigma^2} \int_{-\infty}^{+\infty}e^{-\frac{\sigma^2 h^2}{2\delta^2 \kappa^2}} dh - \int_{-\infty}^{+\infty} h e^{-\frac{\sigma^2 h^2}{2 \delta^2 \kappa^2}} dh - \int_{\sqrt{m}}^\infty \left( \frac{\delta^2}{\sigma^2} m - h\right) e^{-\frac{\sigma^2 h^2}{2 \delta^2 \kappa^2}} dh \\
&= \frac{\delta^2 m}{\sigma^2} \frac{\delta\kappa \sqrt{\tau}}{\sigma} - 0 + O\left( m^{-1/2} e^{-\frac{\sigma^2 m}{2\delta^2 \kappa^2}} \right) = \frac{\delta^3 \kappa m \sqrt{\tau}}{\sigma^3} + O(m^{-5}),
\end{align*}
where we use the fact that the exponential term is an even function (squared terms), while $h \mapsto h$ is odd.

Now assume $k \geq 1$. To estimate $J_k^1$, we are going to use the fact that, denoting $w = \frac{\sigma^2 h}{\delta^2 m}$, we have
\begin{align*}
&\left(\frac{\delta^2}{\sigma^2} m-h\right)^{-k} = \frac{\delta^{-2k} m^{-k}}{\sigma^{-2k}} (1-w)^{-k} \\
&= \frac{\sigma^{2k}}{\delta^{2k} m^{k}} \left(1+kw+\frac{k(k+1)}{2}w^2 + \frac{k(k+1)(k+2)}{6}w^3 +S_k(w) \right),
\end{align*}
where $S(w) = O(w^4)$ for $w$ close to 0. For $h \leq \sqrt{m}$, we have $w \leq \frac{\sigma^2}{\varepsilon^2 \sqrt{m}}$. For, say, $m \geq 1$, we can therefore guarantee the existence of a constant $C_k$ such that $\left|S_k(w)\right| \leq C_k w^4$. 

Let us now define the integrals
\begin{align*}
K_\ell &= \int_{-\infty}^{\sqrt{m}} \left( \frac{\sigma^2 h}{\delta^2 m}\right)^\ell e^{-\frac{\sigma^2 h^2}{2 \delta^2 \kappa^2}} dh, \\
K_\infty^k &= \int_{-\infty}^{\sqrt{m}} S_k\left( \frac{\sigma^2 h}{\delta^2 m}\right) e^{-\frac{\sigma^2 h^2}{2 \delta^2 \kappa^2}} dh.
\end{align*}
Note that $J_k^1 = \frac{\sigma^{2k}}{\delta^{2k} m^k} \left( K_0 + kK_1 + \frac{k(k+1)}{2} K_2 + \frac{k(k+1)(k+2)}{6} K_3 + K^k_\infty \right)$. However, note that $|K_\infty^k| \leq C_k K_4$. This means that we can now restrict our focus to $K_\ell$'s, which allows to remove all apparent singularities caused by terms like $\left(\frac{\delta^2}{\sigma^2} m-h\right)^{-k}$.

In particular, we notice that the $K_\ell$'s are equal to their integral over all reals up to an exponentially negligible quantity in $m$. More precisely,
\begin{equation*}
K_\ell = \frac{\sigma^{2\ell}}{\delta^{2\ell} m^\ell}  \int_{-\infty}^{+\infty} h^\ell e^{-\frac{\sigma^2 h^2}{2 \delta^2 \kappa^2}} dh - \frac{\sigma^{2\ell}}{\delta^{2\ell} m^\ell}  \int_{\sqrt{m}}^{+\infty} h^\ell e^{-\frac{\sigma^2 h^2}{2 \delta^2 \kappa^2}} dh.
\end{equation*}
The second term is then $O\left( m^{\ell/2-1} e^{-\frac{\sigma^2 m}{2\delta^2 \kappa^2}} \right) = O(m^{-5})$.

Now the first integral is proportional to the  moments of the normal distribution, it is thus zero for odd values of $\ell$. For $\ell = 0$, the integral equals $\frac{\delta \kappa \sqrt{\tau}}{\sigma}$. For $\ell =2$, it equals $\frac{\delta \kappa \sqrt{\tau}}{\sigma} \frac{\delta^2 \kappa^2}{\sigma^2} = \frac{\delta^3 \kappa^3 \sqrt{\tau}}{\sigma^3}$. There are formulas for larger values of $\ell$, but for our purposes, it suffices to note that, for $\ell = 4$, the integral is $O(1)$, i.e. it is a constant independent from $m$.

Finally, we can compute the asymptotic approximations of $I_k$'s up to $O(m^{-5})$. Namely,
\begin{align*}
I_{-1} &= \frac{\delta^3 \kappa m \sqrt{\tau}}{\sigma^3} + O(m^{-5}) = \frac{\delta^3 \kappa m \sqrt{\tau}}{\sigma^3} \left( 1+ O(m^{-6}) \right) \\
I_1 &= \frac{\sigma^2}{\delta^2 m} \left( \frac{\delta \kappa \sqrt{\tau}}{\sigma} + \frac{\sigma^4}{\delta^4 m^2} \frac{\delta^3 \kappa^3 \sqrt{\tau}}{\sigma^3}  + O(m^{-4}) \right) \\
&= \frac{\delta^3 \kappa m \sqrt{\tau}}{\sigma^3} \left( \frac{\sigma^4}{\delta^4 m^2} + \frac{\sigma^6 \kappa^2}{\delta^6 m^4} + O(m^{-6}) \right) \\
I_3 &= \frac{\sigma^6}{\delta^6 m^3} \left( \frac{\delta \kappa \sqrt{\tau}}{\sigma} + O(m^{-2}) \right)  \\
&= \frac{\delta^3 \kappa m \sqrt{\tau}}{\sigma^3} \left( \frac{\sigma^8}{\delta^8 m^4} + O(m^{-6}) \right) \\
I_5 &= O(m^{-5}) =  \frac{\delta^3 \kappa m \sqrt{\tau}}{\sigma^3} O(m^{-6}).
\end{align*}

Combining it all finally yields
\begin{align*}
\int_{-\infty}^{m/\gamma} &\int_{m-v}^\infty u^2  e^{-\frac{u^2}{2\delta^2}} du e^{-\frac{v^2}{2\kappa^2}} dv = \delta^2 e^{-\frac{m^2}{2\sigma^2}} (I_{-1} + \delta^2 I_1 - \delta^4 I_3 + I_R) \\
&= \frac{\delta^5 \kappa m \sqrt{\tau} e^{-\frac{m^2}{2 \sigma^2}}}{\sigma^3} \left( 1 + \frac{\sigma^4}{\delta^2 m^2} - \frac{\sigma^6}{\delta^2 m^4} + O(m^{-6}) \right),
\end{align*}
which is the lemma.
\end{proof}

\subsection{Proof of Theorem \ref{th:normal}}

From now on, to simplify notations, we shall denote $\sigma^2 = 1+\varepsilon^2$. In particular $\sigma$ can be interpreted as the standard deviation of the measure $M = G+\xi$. Here is a relation between $\sigma$ and $\varepsilon$ that will be useful later on.

\begin{lemma}
$\mathbb P[M \geq m] = \frac{\sigma e^{-\frac{m^2}{2 \sigma^2}}}{m \sqrt{\tau}} \left( 1 - \frac{\sigma^2}{m^2} + \frac{3\sigma^2}{m^4} + O(m^{-6}) \right)$, as $m \rightarrow \infty$.
\label{lemma:alpha_normal}
\end{lemma}

\begin{proof}
Since $G$ and $\xi$ follow independent normal distributions $\mathcal N(0,1)$ and $\mathcal N(0,\varepsilon^2)$, their sum $M=G+\xi$ follows distribution $\mathcal N(0, \sigma^2)$. Therefore $\mathbb P[M \geq m]$ equals $\frac{1}{\sigma \sqrt{\tau}}$ times the same integral as in the previous lemma.
\end{proof}

Given that $\alpha = \mathbb P[M \geq m_\alpha]$, the form of the previous lemma that will be most often useful is actually
\begin{equation*}
\Delta = \frac{\sigma e^{-\frac{m_\alpha^2}{2 \sigma^2}}}{\alpha m_\alpha \sqrt{\tau}} = 1+\frac{\sigma^2}{m_\alpha^2} - \frac{2 \sigma^4}{m_\alpha^4} + O(m_\alpha^{-6}).
\end{equation*}

\begin{lemma}
$\mathbb E_\alpha[\xi] = \frac{\varepsilon^2 m_\alpha}{\sigma^2} \Delta$, $\mathbb E_\alpha[G] = \frac{m_\alpha}{\sigma^2} \Delta$ and $\mathbb E_\alpha[G\xi] = \frac{\varepsilon^2 m_\alpha^2}{\sigma^4} \Delta$. Also,
\begin{align*}
\mathbb E_\alpha[\xi^2] &= \frac{\varepsilon^4}{\sigma^4} m_\alpha^2 \Delta (1+ \frac{\sigma^4}{\varepsilon^2 m^2} - \frac{\sigma^6}{\varepsilon^2 m^4} + O(m^{-6})), \\
\mathbb E_\alpha[G^2] &= \frac{\Delta m_\alpha^2}{\sigma^4} (1+ \frac{\sigma^4}{m^2} - \frac{\sigma^6}{m^4} + O(m^{-6})).
\end{align*}
\end{lemma}

\begin{proof}
For the three first equalities, note that 
\begin{equation*}
\int_m^\infty te^{-\frac{t^2}{2\sigma^2}} dt = -\sigma^2 \left[ e^{-\frac{t^2}{2\sigma^2}} \right]_m^\infty = -\sigma^2 e^{-\frac{m^2}{2\sigma^2}}.
\end{equation*}
From this on, it is a basic computation.

For the other two inequalities, we apply Lemma \ref{lemma:main_normal} with the appropriate values of $\delta$ and $\kappa$. To compute $\mathbb E_\alpha[\xi^2]$, we use $\kappa = 1$ and $\delta = \varepsilon$. To compute $\mathbb E_\alpha[G^2]$, we use $\kappa = \varepsilon$ and $\delta = 1$.
\end{proof}

\begin{lemma}
As a result, we have
\begin{align*}
Var_\alpha(G) &= \frac{\varepsilon^2}{\sigma^2} + \frac{1}{m^2} + O(m^{-4}), \\
Var_\alpha(\xi) &= \frac{\varepsilon^2}{\sigma^2} + \frac{ \varepsilon^4}{m^2} + O(m^{-4}), \\
Cov_\alpha(G,\xi) &= - \frac{\varepsilon^2}{\sigma^2} + \frac{ \varepsilon^2}{m^2} + O(m^{-4}), \\
Cov_\alpha(M,G) &= \frac{\sigma^2}{m^2} + O(m^{-4}),\\
Var_\alpha(M) &= \frac{\sigma^4}{m^2} + O(m^{-4}). 
\end{align*}
\end{lemma}

\begin{proof}
This is obtained by adding up the results of the previous lemma. For the first line, equal to $\mathbb E_\alpha[G^2] - \mathbb E_\alpha[G]^2$, we have
\begin{align*}
Var_\alpha(G) &= \frac{m^2 \Delta}{\sigma^4} \left( 1 + \frac{\sigma^4}{m^2} - \frac{\sigma^6}{m^4} -1 - \frac{\sigma^2}{m^2} + \frac{2\sigma^4}{m^4} + O(m^{-6}) \right) \\
 &= \frac{m^2 \Delta}{\sigma^4} \left( \frac{\sigma^2 \varepsilon^2}{m^2} + \frac{\sigma^4(2-\sigma^2)}{m^4} + O(m^{-6}) \right) \\
 &= \frac{\varepsilon^2}{\sigma^2} \left( 1+ \frac{\sigma^2}{m^2} + O(m^{-4}) \right) \left( 1 + \frac{\sigma^2(2-\sigma^2)}{\varepsilon^2 m^2} + O(m^{-4}) \right) \\
 &= \frac{\varepsilon^2}{\sigma^2} \left( 1+\frac{\sigma^2}{\varepsilon^2 m^2} + O(m^{-4}) \right),
\end{align*}
which is the first line of the lemma. Moving on to the second one, equal to $\mathbb E_\alpha[\xi^2] - \mathbb E_\alpha[\xi]^2$:
\begin{align*}
Var_\alpha(\xi) &= \frac{\varepsilon^4 m^2 \Delta}{\sigma^4} \left( 1 + \frac{\sigma^4}{\varepsilon^2 m^2} - \frac{\sigma^6}{\varepsilon^2 m^4} -1 - \frac{\sigma^2}{m^2} + \frac{2\sigma^4}{m^4} + O(m^{-6}) \right) \\
 &= \frac{\varepsilon^4 m^2 \Delta}{\sigma^4} \left( \frac{\sigma^2}{\varepsilon^2 m^2} + \frac{\sigma^4(2\varepsilon^2-\sigma^2)}{\varepsilon^2 m^4} + O(m^{-6}) \right) \\
 &= \frac{\varepsilon^2}{\sigma^2} \left( 1+ \frac{\sigma^2}{m^2} + O(m^{-4}) \right) \left( 1 + \frac{\sigma^2(2\varepsilon^2-\sigma^2)}{m^2} + O(m^{-4}) \right) \\
 &= \frac{\varepsilon^2}{\sigma^2} \left( 1+\frac{\varepsilon^2 \sigma^2}{m^2} + O(m^{-4}) \right).
\end{align*}
Now to the third line, which is $\mathbb E_\alpha[G\xi] - \mathbb E_\alpha[G]\mathbb E_\alpha[\xi]$:
\begin{align*}
Cov_\alpha(G,\xi) &= \frac{\varepsilon^2 m^2 \Delta}{\sigma^4} \left( 1 -1 - \frac{\sigma^2}{m^2} + \frac{2\sigma^4}{m^4} + O(m^{-6}) \right) \\
&= - \frac{\varepsilon^2}{\sigma^2} \left( 1+ \frac{\sigma^2}{m^2} + O(m^{-4}) \right) \left(  1 - \frac{2\sigma^2}{m^2} + O(m^{-4}) \right) \\
&= - \frac{\varepsilon^2}{\sigma^2} \left( 1 - \frac{\sigma^2}{m^2} + O(m^{-4}) \right).
\end{align*}
The fourth line is obtained by adding the first and third lines, while the last line is obtained by adding the two first lines and twice the third line.
\end{proof}

\begin{proof}[Proof of Theorem \ref{th:normal}]
Combining it all finally yields
\begin{equation*}
\rho_\alpha \sim \frac{\frac{\sigma^2}{m_\alpha^2}}{\sqrt{\frac{\varepsilon^2}{\sigma^2} \cdot \frac{\sigma^4}{m_\alpha^2}}} \sim \frac{\sigma}{\varepsilon m_\alpha}.
\end{equation*}
Since $M$ is normal with means 0 and variance $1+\varepsilon^2$, then it is well known that
\begin{equation*}
m_\alpha = \sigma \sqrt{2\ln \frac{1}{\alpha} - \ln \ln \frac{1}{\alpha^2} - \ln \tau} + o(1).
\end{equation*}
Combining it all yields the theorem.
\end{proof}

\section{Bounded goal, power law noise}
\label{app:uniform_power}

\subsection{Preliminary lemmas}

\begin{lemma}
For $n \leq 2$ and $m_\alpha \geq 1+\eta$, we have 
$$\int_{m_\alpha-g}^\infty x^n p^\xi(x) dx = \frac{\beta-1}{\beta-(n+1)} \eta^{\beta-1}  (m_\alpha-g)^{n+1-\beta}.$$
\label{lemma:moment_xi}
\end{lemma}

\begin{proof}
This is a mere computation:
\begin{align*}
\int_{m_\alpha-g}^\infty & x^n p^\xi(x) dx
= (\beta-1) \eta^{\beta-1} \int_{m_\alpha-g}^\infty x^{n-\beta} dx \\
&= (\beta-1) \eta^{\beta-1} \left[ \frac{x^{n+1-\beta}}{n+1-\beta} \right]_{m_\alpha-g}^\infty \\ 
&=  \frac{\beta-1}{\beta-(n+1)} \eta^{\beta-1} (m_\alpha-g)^{n+1-\beta},
\end{align*}
which concludes the proof.
\end{proof}

\begin{lemma}
For $\kappa > 1$, we have 
$$\int_0^1 (m-g)^{-\kappa} dg = \frac{(m-1)^{1-\kappa} - m^{1-\kappa}}{\kappa-1}.$$
\label{lemma:integrate_xi}
\end{lemma}

\begin{proof}
$\int_0^1 (m-g)^{-\kappa} dg = \left[ \frac{(m-g)^{1-\kappa}}{\kappa-1} \right]_0^1 = \frac{(m-1)^{1-\kappa} - m^{1-\kappa}}{\kappa-1}.$
\end{proof}

\begin{lemma}
For $\kappa > 2$, we have 
$$\int_0^1 g(m-g)^{-\kappa} dg = \frac{(m-1)^{1-\kappa}}{\kappa-1} - \frac{(m-1)^{2-\kappa} - m^{2-\kappa}}{(\kappa-1)(\kappa-2)}.$$
\label{lemma:integrate_g}
\end{lemma}

\begin{proof}
Using integration by parts, we have
\begin{align*}
(\kappa-1)&(\kappa-2) \int_0^1 g(m-g)^{-\kappa} dg \\
&= (\kappa-1)(\kappa-2) \left( \left[ g \frac{(m-g)^{1-\kappa}}{\kappa-1} \right]_0^1 - \int_0^1 \frac{(m-g)^{1-\kappa}}{\kappa-1} dg \right) \\
&= (\kappa-2) (m-1)^{1-\kappa} + (\kappa-2) \left[ \frac{(m-g)^{2-\kappa}}{\kappa-2} \right]_0^1 \\
&= (\kappa-2) (m-1)^{1-\kappa} - (m-1)^{2-\kappa} + m^{2-\kappa},
\end{align*}
which concludes the proof.
\end{proof}

\begin{lemma}
For $\kappa > 3$, we have 
$$\int_0^1 g^2(m-g)^{-\kappa} dg = \frac{(m-1)^{1-\kappa}}{\kappa-1} - 2 \frac{(m-1)^{2-\kappa}}{(\kappa-1)(\kappa-2)} + 2 \frac{(m-1)^{3-\kappa} -  m^{3-\kappa}}{(\kappa-1)(\kappa-2)(\kappa-3)}.$$
\label{lemma:integrate_g2}
\end{lemma}

\begin{proof}
Using integration by parts, we have
\begin{align*}
(&\kappa-1)(\kappa-2) (\kappa-3) \int_0^1 g^2(m-g)^{-\kappa} dg \\
&= (\kappa-2) (\kappa-3) \left( \left[ g^2 (m-g)^{1-\kappa} \right]_0^1 - 2 \int_0^1 g (m-g)^{1-\kappa} dg \right) \\
&= (\kappa-2) (\kappa-3) (m-1)^{1-\kappa} - 2 (\kappa-3) (m-1)^{2-\kappa} + 2(m-1)^{3-\kappa} - 2m^{3-\kappa},
\end{align*}
using the previous lemma in the last line.
\end{proof}

\subsection{The regime $m_\alpha \geq 1+\eta$}

\begin{lemma}
If $m_\alpha \geq 1+\eta$, then $\alpha(\beta-2) = \eta^{\beta-1} \left( (m_\alpha-1)^{2-\beta} - m_\alpha^{2-\beta} \right)$.
\end{lemma}

\begin{proof}
$\alpha = \mathbb P[M \geq m_\alpha] = \int_0^1 \int_{m_\alpha-g}^\infty p^\xi dg = \eta^{\beta-1} \int_0^1 (m_\alpha-g)^{1-\beta} dg$. We then apply Lemma 12 with $\kappa=\beta-1$.
\end{proof}

\begin{lemma}
For $m_\alpha \geq 1+\eta$, we have the following equalities:
\begin{align*}
\mathbb E_\alpha&[\xi] = \frac{\eta^{\beta-1}}{\alpha} \frac{\beta-1}{(\beta-2)(\beta-3)} \left( (m_\alpha-1)^{3-\beta} - m_\alpha^{3-\beta} \right) \\
\mathbb E_\alpha&[\xi^2] = \frac{\eta^{\beta-1}}{\alpha} \frac{\beta-1}{(\beta-3)(\beta-4)} \left( (m_\alpha-1)^{4-\beta} - m_\alpha^{4-\beta} \right) \\
\mathbb E_\alpha&[G] = \frac{\eta^{\beta-1}}{\alpha (\beta-2)} \left( (m_\alpha-1)^{2-\beta} - \frac{(m_\alpha-1)^{3-\beta} - m_\alpha^{3-\beta}}{\beta-3} \right) \\
\mathbb E_\alpha&[G^2] = \frac{\eta^{\beta-1}}{\alpha (\beta-2)} \left( (m_\alpha-1)^{2-\beta} - \frac{2(m_\alpha-1)^{3-\beta}}{\beta-3} + 2 \frac{(m_\alpha-1)^{4-\beta} - m_\alpha^{4-\beta}}{(\beta-3)(\beta-4)} \right) \\
\mathbb E_\alpha&[G \xi] = \frac{\eta^{\beta-1}}{\alpha} \frac{\beta-1}{(\beta-2)(\beta-3)} \left( (m_\alpha-1)^{3-\beta} - \frac{(m_\alpha-1)^{4-\beta} - m_\alpha^{4-\beta}}{\beta-4} \right)
\end{align*}
\end{lemma}

\begin{proof}
This follows from combining the equations in appendix to the previous lemmas. More precisely, 
\begin{itemize}
\item $\mathbb E_\alpha[\xi]$ is obtained by Lemma \ref{lemma:moment_xi} for $n=1$ and Lemma  \ref{lemma:integrate_xi} for $\kappa = \beta-2$.
\item $\mathbb E_\alpha[\xi^2]$ is obtained by Lemma \ref{lemma:moment_xi} for $n=2$ and Lemma  \ref{lemma:integrate_xi} for $\kappa = \beta-3$.
\item $\mathbb E_\alpha[G]$ is obtained by Lemma \ref{lemma:moment_xi} for $n=0$ and Lemma  \ref{lemma:integrate_g} for $\kappa = \beta-1$.
\item $\mathbb E_\alpha[G^2]$ is obtained by Lemma \ref{lemma:moment_xi} for $n=0$ and Lemma  \ref{lemma:integrate_g2} for $\kappa = \beta-1$.
\item $\mathbb E_\alpha[G\xi]$ is obtained by Lemma \ref{lemma:moment_xi} for $n=1$ and Lemma  \ref{lemma:integrate_g} for $\kappa = \beta-2$.
\end{itemize}
This concludes the proof.
\end{proof}

\subsection{The regime $\eta \leq m_\alpha \leq 1+\eta$}

\begin{lemma}
If $\eta \leq m_\alpha \leq 1+\eta$, then 
$$\alpha = 1-(m_\alpha-\eta) + \frac{\eta - \eta^{\beta-1} m_\alpha^{2-\beta}}{\beta-2}.$$
\end{lemma}

\begin{proof}
Using the appendix and a variant of Lemma \ref{lemma:integrate_xi} for $\kappa = \beta-1$, we have
\begin{align*}
\alpha &= \mathbb P[M \geq m_\alpha] = \int_0^{m_\alpha-\eta} \int_{m_\alpha-g}^\infty p^\xi dg + \int_{m_\alpha-\eta}^1 dg \\
&= \eta^{\beta-1} \int_0^{m_\alpha-\eta} (m_\alpha-g)^{1-\beta} dg + \left[ g \right]_{m_\alpha-\eta}^1 \\
&= \eta^{\beta-1} \frac{\eta^{2-\beta} - m_\alpha^{2-\beta}}{\beta-2} + 1-(m_\alpha-\eta),
\end{align*}
which is the lemma.
\end{proof}

\begin{lemma}
If $\eta \leq m_\alpha \leq 1+\eta$, then 
\begin{align*}
\mathbb E_\alpha[\xi] &= \frac{1}{\alpha} \frac{\beta-1}{\beta-2} \left( \eta - \eta(m_\alpha - \eta) + \frac{\eta^{\beta-1} (\eta^{3-\beta}- m_\alpha^{3-\beta})}{\beta-3} \right), \\
\mathbb E_\alpha[\xi^2] &= \frac{1}{\alpha} \frac{\beta-1}{\beta-3} \left( \eta^2 - \eta^2(m_\alpha-\eta) + \eta^{\beta-1} \frac{\eta^{4-\beta}-m_\alpha^{4-\beta}} {\beta-4} \right), \\
\mathbb E_\alpha[G \xi] &= \frac{1}{\alpha} \frac{\beta-1}{\beta-2} \left( \frac{\eta}{2} - \frac{\eta (m_\alpha-\eta)^2}{2} + \frac{\eta^2 (m_\alpha-\eta)}{\beta-3} - \frac{\eta^{\beta-1} (\eta^{4-\beta} - m_\alpha^{4-\beta}) }{(\beta-3)(\beta-4)} \right), \\
\mathbb E_\alpha[G] &= \frac{1}{\alpha} \left( \frac{1-(m_\alpha-\eta)^2}{2} + \frac{\eta (m_\alpha-\eta)} {\beta-2} - \frac{\eta^{\beta-1} (\eta^{3-\beta}- m_\alpha^{3-\beta}}{(\beta-2)(\beta-3)} \right), \\
\mathbb E_\alpha[G^2] &= \frac{1}{\alpha} \left( \frac{1-(m_\alpha-\eta)^3}{3} + \frac{\eta (m_\alpha-\eta)^2}{\beta-2} - \frac{2 \eta^2 (m_\alpha-\eta)}{(\beta-2)(\beta-3)} + \frac{2\eta^{\beta-1}( \eta^{4-\beta} - m_\alpha^{4-\beta})}{(\beta-2)(\beta-3)(\beta-4)} \right).
\end{align*}
\end{lemma}

\begin{proof}
Using the appendix and a variant of Lemma \ref{lemma:integrate_xi} for $\kappa = \beta-1$, we have
\begin{align*}
\alpha & \frac{\beta-2}{\beta-1} \mathbb E_\alpha[\xi] = \frac{\beta-2}{\beta-1} \int_0^1 \int_{\max(m_\alpha-g,\eta)}^\infty x p^\xi(x) dx dg \\
&= \eta^{\beta-1} \int_0^{m_\alpha-\eta} (m_\alpha-g)^{2-\beta} dg + \int_{m_\alpha-\eta}^1 \eta dg \\
&= \eta^{\beta-1} \left[ \frac{(m_\alpha-g)^{3-\beta}}{\beta-3} \right]_0^{m_\alpha-\eta} + \eta (1-(m_\alpha-\eta)) \\
&= \eta - \eta (m_\alpha - \eta) + \eta^{\beta-1} \frac{\eta^{3-\beta} - m_\alpha^{3-\beta}}{\beta-3},
\end{align*}
from which we can conclude by factorizing by $m_\alpha-\eta$. Moving on to the second line:
\begin{align*}
\alpha & \frac{\beta-3}{\beta-1} \mathbb E_\alpha[\xi^2] = \frac{\beta-3}{\beta-1} \int_0^1 \int_{\max(m_\alpha-g,\eta)}^\infty x^2 p^\xi(x) dx dg \\
&= \eta^{\beta-1} \int_0^{m_\alpha-\eta} (m_\alpha-g)^{3-\beta} dg + \int_{m_\alpha-\eta}^1 \eta^2 dg \\
&= \eta^{\beta-1} \left[ \frac{(m_\alpha-g)^{4-\beta}}{\beta-4} \right]_0^{m_\alpha-\eta} + \eta^2 (1-( m_\alpha-\eta)) \\
&= \eta^2 - \eta^2(m_\alpha-\eta) + \eta^{\beta-1} \frac{\eta^{4-\beta}-m_\alpha^{4-\beta}} {\beta-4},
\end{align*}
which allows to conclude. Finally, we have the case of $\mathbb E_\alpha[G\xi]$:
\begin{align*}
\alpha & \frac{\beta-2}{\beta-1} \mathbb E_\alpha[G \xi] 
= \eta^{\beta-1} \int_0^{m_\alpha-\eta} (m_\alpha-g)^{2-\beta} g dg + \int_{m_\alpha-\eta}^1 \eta g dg \\
&= \eta^{\beta-1} \left[ \frac{g (m_\alpha-g)^{3-\beta}}{\beta-3} \right]_0^{m_\alpha-\eta} - \frac{\eta^{\beta-1}}{\beta-3} \int_0^{m_\alpha-\eta} (m_\alpha-g)^{3-\beta} dg  +  \eta \int_{m_\alpha-\eta}^1 g dg \\
&= \eta^{\beta-1} (m_\alpha-\eta) \frac{\eta^{3-\beta}}{\beta-3} - \frac{\eta^{\beta-1}}{\beta-3} \left[ \frac{(m_\alpha-g)^{4-\beta}}{\beta-4} \right]_0^{m_\alpha-\eta} + \eta \left[ \frac{g^2}{2} \right]_{m_\alpha-\eta}^1 \\
&= \frac{\eta}{2} - \frac{\eta (m_\alpha-\eta)^2}{2} + \frac{\eta^2 (m_\alpha-\eta)}{\beta-3} - \frac{\eta^{\beta-1} (\eta^{4-\beta} - m_\alpha^{4-\beta}) }{(\beta-3)(\beta-4)},
\end{align*}
which concludes the proof. Using the appendix equations and a variant of Lemma \ref{lemma:integrate_g} for $\kappa = \beta-1$, we have
\begin{align*}
\alpha & \mathbb E_\alpha[G] = \int_0^{m_\alpha-\eta} \int_{m_\alpha-g}^\infty p^\xi g dg + \int_{m_\alpha-\eta}^1 g dg \\
&= \eta^{\beta-1} \int_0^{m_\alpha-\eta} g (m_\alpha-g)^{1-\beta} dg + \frac{1}{2} - \frac{(m_\alpha-\eta)^2}{2}.
\end{align*}
To conclude, note that 
\begin{align*}
(\beta-2) &\int_0^{m_\alpha-\eta} g (m_\alpha-g)^{1-\beta} dg \\
&= \left[ g (m_\alpha-g)^{2-\beta} \right]_0^{m_\alpha-\eta} - \int_0^{m_\alpha-\eta} (m_\alpha-g)^{2-\beta} dg \\
&= \eta^{2-\beta} (m_\alpha-\eta) - \frac{\eta^{3-\beta} - m_\alpha^{3-\beta}}{\beta-3} .
\end{align*}
Rearranging the terms yields the lemma.  Using the appendix equations and a variant of Lemma \ref{lemma:integrate_g} for $\kappa = \beta-1$, we have
\begin{align*}
\alpha & \mathbb E_\alpha[G^2] = \int_0^{m_\alpha-\eta} \int_{m_\alpha-g}^\infty p^\xi g^2 dg + \int_{m_\alpha-\eta}^1 g^2 dg \\
&= \eta^{\beta-1} \int_0^{m_\alpha-\eta} g^2 (m_\alpha-g)^{1-\beta} dg + \frac{1}{3} - \frac{(m_\alpha-\eta)^3}{3}.
\end{align*}
Now let us decompose the remaining integral once:
\begin{align*}
(&\beta-2) \int_0^{m_\alpha-\eta} g^2 (m_\alpha-g)^{1-\beta} dg \\
&= \left[ g^2 (m_\alpha-g)^{2-\beta} \right]_0^{m_\alpha-\eta} - 2 \int_0^{m_\alpha-\eta} g (m_\alpha-g)^{2-\beta} dg \\
&= (m_\alpha-\eta)^2 \eta^{2-\beta} - \frac{2(m_\alpha-\eta) \eta^{3-\beta}}{\beta-3}  + \frac{2}{\beta-3} \int_0^{m_\alpha-\eta} (m_\alpha-g)^{3-\beta} dg.
\end{align*}
Note that this last integral is $\left[ \frac{(m_\alpha-g)^{4-\beta}}{\beta-4}\right]_0^{m_\alpha-\eta} = \frac{\eta^{4-\beta} - m_\alpha^{4-\beta}}{\beta-4}$.
Rearranging the terms yields the lemma.
\end{proof}

This allows to plot $E[G|M \geq m_\alpha]$ as a function of $m_\alpha$. This is what we did in Figure \ref{fig:G_given_M_uniform_power}. Combining all the results of the last two sections, we can finally compute the correlation as a function of $m_\alpha$. This is what is done in Figure \ref{fig:correlation_uniform_power}.

\subsection{The limit $\alpha \rightarrow 0$}

\begin{lemma}
As $m\rightarrow \infty$, we have 
$$\frac{(m-1)^{-\gamma} - m^{-\gamma}}{\gamma} = m^{-\gamma-1} + \frac{\gamma+1}{2} m^{-\gamma-2}+\frac{(\gamma+1)(\gamma+2)}{6} m^{-\gamma-3} + o(m^{-\gamma-3}).$$
\end{lemma}

\begin{proof}
Note that 
\begin{align*}
(&m-1)^{-\gamma} - m^{-\gamma} = m^{-\gamma} (1-m^{-1})^{-\gamma} - m^{-\gamma} \\
&= m^{-\gamma} \left( 1+ \gamma m^{-1} + \frac{\gamma(\gamma+1)}{2} m^{-2} + \frac{\gamma(\gamma+1)(\gamma+2)}{6} m^{-3} + o(m^{-3}) \right) - m^{-\gamma} \\
&= \gamma m^{-\gamma-1} + \frac{\gamma(\gamma+1)}{2} m^{-\gamma-2} + \frac{\gamma(\gamma+1)(\gamma+2)}{6} m^{-\gamma-3} + o(m^{-\gamma-3}),
\end{align*}
which concludes the proof.
\end{proof}

\begin{lemma}
$m_\alpha \sim \eta \alpha^{-\frac{1}{\beta-1}}$, as $\alpha \rightarrow 0$.
\end{lemma}

\begin{proof}
First notice that
\begin{equation*}
\mathbb P\left[ M \geq 1+\left( \frac{\eta^{\beta-1}}{(\beta-1)\alpha} \right)^{\frac{1}{\beta-2}} \right] \geq \frac{\eta^{\beta-1}}{\beta-2} \frac{(\beta-1)\alpha}{\eta^{\beta-1}} = \alpha.
\end{equation*}
This implies that $m_\alpha \geq 1+\left( \frac{\eta^{\beta-1}}{(\beta-1)\alpha} \right)^{\frac{1}{\beta-2}} \rightarrow \infty$, as $\alpha \rightarrow 0$. 

For $\alpha$ small enough, we thus have $\mathbb P[M \geq m_\alpha] =\alpha = \frac{\eta^{\beta-1}}{\beta-2} ((m_\alpha-1)^{2-\beta} - m_\alpha^{2-\beta})$. We can thus write 
$$(1-m_\alpha^{-1})^{2-\beta} = 1-(2-\beta) m_\alpha^{-1} + o(m_\alpha^{-1}).$$
Applying this to the above equation implies $m_\alpha \sim \eta \alpha^{-\frac{1}{\beta-1}}$. 
\end{proof}

\begin{lemma}
As $\alpha \rightarrow 0$, we have the following expansions:
\begin{align*}
\mathbb E_\alpha[\xi] &= \frac{\beta-1}{\beta-2} m_\alpha - \frac{\beta-1}{2(\beta-2)} + o(1) \\
\mathbb E_\alpha[\xi^2] &= \frac{\beta-1}{\beta-3} m_\alpha^2 + o(m_\alpha^2) \\
\mathbb E_\alpha[G] &= \frac{1}{2} + \frac{\beta-1}{12} m_\alpha^{-1} + o(m_\alpha^{-1}) \\
\mathbb E_\alpha[G^2] &= \frac{1}{3} + o(1) \\
\mathbb E_\alpha[G \xi] &= \frac{\beta-1}{2(\beta-2)} m_\alpha + \frac{(\beta-1)(\beta-5)}{12 (\beta-2)} + o(1).
\end{align*}
\end{lemma}

\begin{proof}
The proofs consist of combining the exact expressions of the expectations computed in a previous lemma, with the asymptotic expansions of the previous lemma.
\begin{align*}
\mathbb E_\alpha[\xi] &= \frac{\beta-1}{\beta-2} \frac{\frac{(m_\alpha-1)^{3-\beta} - m_\alpha^{3-\beta}}{\beta-3}}{\frac{(m_\alpha-1)^{2-\beta} - m_\alpha^{2-\beta}}{\beta-2}} \\
&= \frac{\beta-1}{\beta-2} \frac{m_\alpha^{2-\beta} + \frac{\beta-2}{2} m_\alpha^{1-\beta} + o(m_\alpha^{1-\beta})}{m_\alpha^{1-\beta} + \frac{\beta-1}{2} m_\alpha^{-\beta} + o(m_\alpha^{-\beta})} \\
&= \frac{\beta-1}{\beta-2}  \left( m_\alpha+ \frac{\beta-2}{2} + o(1) \right) \left( 1- \frac{\beta-1}{2} m_\alpha^{-1} + o(m_\alpha^{-1}) \right) \\
&= \frac{\beta-1}{\beta-2} m_\alpha + \frac{\beta-1}{2(\beta-2)} \left( \beta-2 - (\beta-1) \right) + o(1),
\end{align*}
which yields the first asymptotic approximation. For the second, we will not need to go too deep in the approximation:
\begin{align*}
\mathbb E_\alpha[\xi^2] &= \frac{\beta-1}{\beta-3} \frac{\frac{(m_\alpha-1)^{4-\beta} - m_\alpha^{4-\beta}}{\beta-4}}{\frac{(m_\alpha-1)^{2-\beta} - m_\alpha^{2-\beta}}{\beta-2}} \\
&= \frac{\beta-1}{\beta-3} \frac{m_\alpha^{2-\beta} + o(m_\alpha^{2-\beta})}{m_\alpha^{1-\beta} + o(m_\alpha^{1-\beta})} \\
&= \frac{\beta-1}{\beta-3}  m_\alpha^2 + o(m_\alpha^{2}),
\end{align*}
which will suffice. The third one will need more computations.
\begin{align*}
\mathbb E_\alpha[G] &= \frac{1}{\beta-2} \frac{(m_\alpha-1)^{2-\beta} - \frac{(m_\alpha-1)^{3-\beta} - m_\alpha^{3-\beta}}{\beta-3}}{\frac{(m_\alpha-1)^{2-\beta} - m_\alpha^{2-\beta}}{\beta-2}} \\
&= \frac{(m_\alpha-1)^{2-\beta} - (m_\alpha^{2-\beta} + \frac{\beta-2}{2} m_\alpha^{1-\beta} + \frac{(\beta-2)(\beta-1)}{6} m_\alpha^{-\beta} + o(m_\alpha^{-\beta}))}{(\beta-2)(m_\alpha^{1-\beta} + \frac{\beta-1}{2} m_\alpha^{-\beta} + o(m_\alpha^{-\beta}))} \\
&= \frac{m_\alpha^{1-\beta} + \frac{\beta-1}{2} m_\alpha^{-\beta} - \frac{1}{2} m_\alpha^{1-\beta} - \frac{\beta-1}{6} m_\alpha^{-\beta} + o(m_\alpha^{-\beta})}{m_\alpha^{1-\beta} (1+ \frac{\beta-1}{2} m_\alpha^{-1} + o(m_\alpha^{-1}))} \\
&= \left( \frac{1}{2} + \frac{\beta-1}{3} m_\alpha^{-1} + o(m_\alpha^{-1}) \right) \left( 1 - \frac{\beta-1}{2} m_\alpha^{-1} + o(m_\alpha^{-1}) \right) \\
&= \frac{1}{2} + \left( \frac{1}{3} - \frac{1}{4} \right) (\beta-1) m_\alpha^{-1} + o(m_\alpha^{-1}),
\end{align*}
which yields the third line of the lemma. The fourth lemma is a bit tricky. Let us first note that
\begin{align*}
-2&\frac{(m_\alpha-1)^{3-\beta}}{\beta-3} + 2 \frac{(m_\alpha-1)^{4-\beta} - m_\alpha^{4-\beta}}{(\beta-3)(\beta-4)} \\
&= -2 \frac{(m_\alpha-1)^{3-\beta}}{\beta-3} + 2 \frac{m_\alpha^{\beta-3} + \frac{\beta-3}{2} m_\alpha^{2-\beta} + \frac{(\beta-3)(\beta-2)}{6} m_\alpha^{1-\beta}}{\beta-3} + o(m_\alpha^{1-\beta}) \\
&= -2 \frac{(m_\alpha-1)^{3-\beta} - m_\alpha^{\beta-3}}{\beta-3} + m_\alpha^{2-\beta} + \frac{\beta-2}{3} m_\alpha^{1-\beta} + o(m_\alpha^{1-\beta}) \\
&= -2 \left( m_\alpha^{2-\beta} + \frac{\beta-2}{2} m_\alpha^{1-\beta} \right) + m_\alpha^{2-\beta} + \frac{\beta-2}{3} m_\alpha^{1-\beta} + o(m_\alpha^{1-\beta}) \\
&= -m_\alpha^{2-\beta} - \frac{2(\beta-2)}{3} m_\alpha^{1-\beta} + o(m_\alpha^{1-\beta}).
\end{align*}
We can plug this into the equation of $\mathbb E_\alpha[G^2]$, which yields
\begin{align*}
\mathbb E_\alpha&[G^2] = \frac{ \frac{(m_\alpha-1)^{2-\beta} - m_\alpha^{2-\beta}}{\beta-2} - \frac{2}{3} m_\alpha^{1-\beta} + o(m_\alpha^{1-\beta})}{ m_\alpha^{1-\beta} + o(m_\alpha^{1-\beta})} \\
&= \frac{m_\alpha^{1-\beta} - \frac{2}{3} m_\alpha^{1-\beta}}{m_\alpha^{1-\beta}} + o(1) = \frac{1}{3} + o(1).
\end{align*}
Finally, we move on to the last asymptotic approximation. First note that
\begin{align*}
&\frac{(m_\alpha-1)^{3-\beta}}{\beta-3} - \frac{(m_\alpha-1)^{4-\beta} - m_\alpha^{4-\beta}}{(\beta-3)(\beta-4)} \\
&= \frac{(m_\alpha-1)^{3-\beta} - m_\alpha^{3-\beta}}{\beta-3} - \frac{1}{2} m_\alpha^{2-\beta} - \frac{\beta-2}{6} m_\alpha^{1-\beta} 
+o(m_\alpha^{1-\beta}) \\
&= m_\alpha^{2-\beta} + \frac{\beta-2}{2} m_\alpha^{1-\beta} 
- \frac{1}{2} m_\alpha^{2-\beta}  - \frac{\beta-2}{6} m_\alpha^{1-\beta} 
+o(m_\alpha^{1-\beta}) \\
&= \frac{1}{2} m_\alpha^{2-\beta} + \frac{\beta-2}{3} m_\alpha^{1-\beta} + o(m_\alpha^{1-\beta}).
\end{align*}
As a result,
\begin{align*}
&\mathbb E_\alpha[G\xi] = (\beta-1) \frac{\frac{1}{2} m_\alpha^{2-\beta} + \frac{\beta-2}{3} m_\alpha^{1-\beta} + o(m_\alpha^{1-\beta})}{(\beta-2) m_\alpha^{1-\beta} \left( 1  + \frac{\beta-1}{2} m_\alpha^{-1} + o(m_\alpha^{-1}) \right)} \\
&= \frac{\beta-1}{2(\beta-2)} m_\alpha \left( 1+ \frac{2(\beta-2)}{3} m_\alpha^{-1} + o(m_\alpha^{-1}) \right) \left( 1 - \frac{\beta-1}{2} m_\alpha^{-1} + o(m_\alpha^{-1}) \right) \\
&= \frac{\beta-1}{2(\beta-2)} m_\alpha + \frac{\beta-1}{2(\beta-2)} \left( \frac{2(\beta-2)}{3} - \frac{\beta-1}{2} \right) + o(1),
\end{align*}
which yields the last approximation of the lemma.
\end{proof}

\begin{lemma}
As a result, we have $Var_\alpha (G) \rightarrow 1/12$, $Cov_\alpha(G,M) = - \frac{1}{12 (\beta-2)}$ and $Var_\alpha (M) \sim \frac{\beta-1}{(\beta-2)^2 (\beta-3)} m_\alpha^2$.
\end{lemma}

\begin{proof}
We have $Var_\alpha(G) = \mathbb E_\alpha[G^2] - \mathbb E_\alpha[G]^2 = \frac{1}{3}-\frac{1}{4} + o(1) = \frac{1}{12} + o(1)$. Next, note that
\begin{align*}
\mathbb E_\alpha&[G] \mathbb E_\alpha[\xi] = \left( \frac{1}{2} + \frac{\beta-1}{12} m_\alpha^{-1} + o(m_\alpha^{-1}) \right) \left( \frac{\beta-1}{\beta-2} m_\alpha - \frac{\beta-1}{2(\beta-2)} + o(1) \right) \\
&= \frac{\beta-1}{2(\beta-2)} m_\alpha + \frac{(\beta-1)^2}{12(\beta-2)} - \frac{\beta-1}{4(\beta-2)} + o(1) \\
&= \frac{\beta-1}{2(\beta-2)} m_\alpha + \frac{(\beta-1)}{12(\beta-2)} \left( \beta- 4 \right) + o(1).
\end{align*}
As a result, we have
\begin{align*}
Cov_\alpha(G,\xi) \rightarrow \frac{\beta-1}{12(\beta-2)} \left( (\beta-5) - (\beta-4) \right) = -\frac{\beta-1}{12(\beta-2)}.
\end{align*}
From this, we derive the fact that
\begin{align*}
Cov_\alpha(G,M) &= Var_\alpha(G) + Cov_\alpha(G,\xi) \\
&\rightarrow \frac{1}{12} \left( 1 - \frac{\beta-1}{\beta-2} \right) =  - \frac{1}{12 (\beta-2)}.
\end{align*}
Note also that
\begin{align*}
Var_\alpha(\xi) &\sim \left( \frac{1}{\beta-3} - \frac{\beta-1}{(\beta-2)^2}\right) (\beta-1) m_\alpha^2 \\
&\sim \frac{\beta-1}{(\beta-2)^2(\beta-3)} m_\alpha^2
\end{align*}
This leads to
\begin{align*}
Var(M) &= Var_\alpha(G) + 2 Cov_\alpha(G,\xi) + Var_\alpha(\xi) \\
&= \frac{\beta-1}{(\beta-2)^2(\beta-3)} m_\alpha^2 + o(m_\alpha^2),
\end{align*}
which conclucdes the lemma.
\end{proof}

\begin{proof}[Proof of Theorem \ref{th:uniform_power}]
Finally, we can prove the theorem. Combining it all, we have
\begin{equation*}
\rho_\alpha \sim \frac{\frac{-1}{12(\beta-2)}}{\sqrt{\frac{\beta-1}{12(\beta-2)^2(\beta-3)} m_\alpha^2}} \sim \sqrt{\frac{\beta-3}{12(\beta-1)}} m_\alpha^{-1}
\end{equation*}
Using the asymptotic approximation $m_\alpha^{1-\beta} = \alpha \eta^{1-\beta}$ allows to conclude.
\end{proof}

\subsection{What happens when $\alpha \approx \varepsilon$}

As $\alpha$ goes to zero, the correlation $\rho_\alpha$, which is first positive, and even near 1, then becomes equivalent to, roughly, $-\frac{\alpha^{\frac{1}{\beta-1}}}{\varepsilon}$. This means that it must become negative at some point, before then converging to zero. In this subsection, we study this turning point. More precisely, we will focus on what happens when $m_\alpha = 1+\eta$, 

To do so, we shall assume that $\alpha(\varepsilon)$ is such that $m_{\alpha(\varepsilon)} = 1+\eta$. And as $\alpha \rightarrow 0$, we will thus also have $\varepsilon \rightarrow 0$.

\begin{lemma}
$(\beta-2) \alpha(\varepsilon) = \eta \left( 1-\eta^{\beta-2} (1+\eta)^{2-\beta} \right)$.
\end{lemma}

\begin{proof}
Indeed, $\mathbb P[M\geq 1+\eta] = \frac{\eta^{\beta-1}}{\beta-2} \left(\eta^{2-\beta} - (\eta+1)^{2-\beta} \right)$.
\end{proof}

\begin{lemma}
$\alpha(\varepsilon) \sim \sqrt{\frac{\beta-3}{\beta-1}} \varepsilon$, as $\varepsilon \rightarrow 0$.
\end{lemma}

\begin{proof}
This is an immediate corollary of the previous lemma, as we recall that $\eta = (\beta-2) \sqrt{\frac{\beta-3}{\beta-1}} \varepsilon$.
\end{proof}

\begin{lemma}
We have the following equalities:
\begin{align*}
\mathbb E_{\alpha(\varepsilon)}&[\xi] = \frac{\eta^2}{\alpha(\varepsilon)} \frac{\beta-1}{\beta-2} \frac{ 1 - \eta^{\beta-3} (1+\eta)^{3-\beta} } {\beta-3}, \\
\mathbb E_{\alpha(\varepsilon)}&[\xi^2] = \frac{\eta^3}{\alpha(\varepsilon)} \frac{\beta-1}{\beta-3} \frac{ 1 - \eta^{\beta-4} (1+\eta)^{4-\beta} } {\beta-4}, \\
\mathbb E_{\alpha(\varepsilon)}&[G] = \frac{\eta^{\beta-1}}{\alpha(\varepsilon) (\beta-2)} \left( \eta^{2-\beta} - \frac{\eta^{3-\beta} - (1+\eta)^{3-\beta}}{\beta-3} \right), \\
\mathbb E_{\alpha(\varepsilon)}&[G^2] = \frac{\eta^{\beta-1}}{\alpha(\varepsilon) (\beta-2)} \left( \eta^{2-\beta} - \frac{2\eta^{3-\beta}}{\beta-3} + 2 \frac{\eta^{4-\beta} - (1+\eta)^{4-\beta}}{(\beta-3)(\beta-4)} \right), \\
\mathbb E_{\alpha(\varepsilon)}&[G \xi] = \frac{\eta^{\beta-1}}{\alpha(\varepsilon)} \frac{\beta-1}{(\beta-2)(\beta-3)} \left( \eta^{3-\beta} - \frac{\eta^{4-\beta} - (1+\eta)^{4-\beta}}{\beta-4} \right).
\end{align*}
\end{lemma}

\begin{lemma}
We have the following equalities:
\begin{align*}
\mathbb E_{\alpha(\varepsilon)}[\xi] &= \eta \frac{\beta-1}{\beta-3} \frac{ 1 - \eta^{\beta-3} (1+\eta)^{3-\beta} } {1-\eta^{\beta-2} (1+\eta)^{2-\beta}} \\
&= \eta \frac{\beta-1}{\beta-3}  + o(\eta), \\
\mathbb E_{\alpha(\varepsilon)}[\xi^2] &= \eta^2 \frac{(\beta-1)(\beta-2)}{(\beta-3)(\beta-4)} \frac{ 1 - \eta^{\beta-4} (1+\eta)^{4-\beta} } {1-\eta^{\beta-2} (1+\eta)^{2-\beta}}, \\
\mathbb E_{\alpha(\varepsilon)}[G] &= \frac{ 1 - \frac{\eta}{\beta-3} + \frac{\eta^{\beta-2}}{\beta-3} (1+\eta)^{3-\beta} } {1-\eta^{\beta-2} (1+\eta)^{2-\beta}} \\
&= 1- \frac{\eta}{\beta-3} + \frac{\beta-2}{\beta-3} \eta^{\beta-2} + o(\eta^{\beta-2}), \\
\mathbb E_{\alpha(\varepsilon)}[G^2] &= \frac{ 1 - \frac{2\eta}{\beta-3} + \frac{2 \eta^2} {(\beta-3)(\beta-4)} (1 - \eta^{\beta-4}(1+\eta)^{4-\beta})} {1-\eta^{\beta-2} (1+\eta)^{2-\beta}}, \\
\mathbb E_{\alpha(\varepsilon)}[G \xi] &= \eta \frac{\beta-1}{\beta-3} \frac{ 1 - \frac{\eta}{\beta-4} (1 - \eta^{\beta-4}(1+\eta)^{4-\beta}) } {1-\eta^{\beta-2} (1+\eta)^{2-\beta}}.
\end{align*}
\end{lemma}

Now, to understand the limit $\eta \rightarrow 0$, we need to distinguish the cases $\beta > 4$ and $\beta < 4$, because the quantity $1-\eta^{\beta-4}(1+\eta)^{4-\beta}$ does not behave the same way in the two cases. Indeed, $1-\eta^{\beta-4}(1+\eta)^{4-\beta} \rightarrow 1$ if $\beta > 4$, while $1-\eta^{\beta-4}(1+\eta)^{4-\beta} \sim - \eta^{\beta-4}$ if $\beta < 4$.

\begin{lemma}
If $\beta > 4$, as $\varepsilon \rightarrow 0$, $\rho_{\alpha(\varepsilon)} \rightarrow - \frac{1}{\beta-2}$.
\end{lemma}

\begin{proof}
Let us compute the expectations of key variables.
\begin{align*}
\mathbb E_{\alpha(\varepsilon)}[\xi^2] &= \eta^2 \frac{(\beta-1)(\beta-2)}{(\beta-3)(\beta-4)} + o(\eta^{\beta-2}), \\
\mathbb E_{\alpha(\varepsilon)}[G^2] &= 1- \frac{2\eta}{\beta-3} + \frac{2 \eta^2}{(\beta-3)(\beta-4)} + o(\eta^2), \\
\mathbb E_{\alpha(\varepsilon)}[G \xi] &= \eta \frac{\beta-1}{\beta-3} \left( 1 - \frac{\eta}{\beta-4} + \frac{\eta^{\beta-3}}{\beta-4} + o(\eta^{\beta-3}) \right).
\end{align*}
We now move on to variances and covariances.
\begin{align*}
Var_{\alpha(\varepsilon)} (\xi) &\sim 2(\beta-1) \frac{\eta^2}{(\beta-3)^2(\beta-4)}, \\
Var_{\alpha(\varepsilon)} (G) &\sim (\beta-2) \frac{\eta^2}{(\beta-3)^2 (\beta-4)} \\
Cov_{\alpha(\varepsilon)} (G,\xi) &\sim - (\beta-1) \frac{\eta^2} {(\beta-3)^2(\beta-4)}, \\
Cov_{\alpha(\varepsilon)} (G,M) &\sim - \frac{\eta^2} {(\beta-3)^2(\beta-4)}, \\
Var_{\alpha(\varepsilon)} (M) &\sim (\beta-2) \frac{\eta^2}{(\beta-3)^2(\beta-4)}.
\end{align*}
Combining it all allows to conclude.
\end{proof}

\begin{lemma}
If $3 < \beta < 4$, as $\varepsilon \rightarrow 0$, $\rho_{\alpha(\varepsilon)} \rightarrow - \sqrt{\frac{\beta-3}{2(\beta-2)}}$.
\end{lemma}

\begin{proof}
Let us compute the expectations of key variables.
\begin{align*}
\mathbb E_{\alpha(\varepsilon)}[\xi^2] &= \eta^{\beta-2} \frac{(\beta-1)(\beta-2)}{(\beta-3)(4-\beta)} + o(\eta^{\beta-2}), \\
\mathbb E_{\alpha(\varepsilon)}[G^2] &= 1- \frac{2\eta}{\beta-3} + \frac{\eta^{\beta-2} (\beta-2) (5-\beta)}{(\beta-3)(4-\beta)} + o(\eta^{\beta-2}), \\
\mathbb E_{\alpha(\varepsilon)}[G \xi] &= \eta \frac{\beta-1}{\beta-3} \left( 1 - \frac{\eta^{\beta-3} (\beta-3)}{4-\beta} + o(\eta^{\beta-3}) \right), \\
\mathbb E_{\alpha(\varepsilon)}[\xi]^2 &= O(\eta^2) = o(\eta^{\beta-2}), \\
\mathbb E_{\alpha(\varepsilon)}[G]^2 &= 1-\frac{2\eta}{\beta-3} + 2 \frac{\beta-2}{\beta-3} \eta^{\beta-2} + o(\eta^{\beta-2}), \\
\mathbb E_{\alpha(\varepsilon)}[G] \mathbb E_{\alpha(\varepsilon)}[\xi] &= \eta \frac{\beta-1}{\beta-3}+O(\eta^2).
\end{align*}
We now move on to variances and covariances.
\begin{align*}
Var_{\alpha(\varepsilon)} (\xi) &\sim \frac{\eta^{\beta-2}}{4-\beta} \frac{(\beta-1)(\beta-2)}{\beta-3}, \\
Var_{\alpha(\varepsilon)} (G) &\sim \frac{\eta^{\beta-2}}{4-\beta} (\beta-2) \\
Cov_{\alpha(\varepsilon)} (G,\xi) &\sim - \frac{\eta^{\beta-2}}{4-\beta} (\beta-1), \\
Cov_{\alpha(\varepsilon)} (G,M) &\sim - \frac{\eta^{\beta-2}}{4-\beta}, \\
Var_{\alpha(\varepsilon)} (M) &\sim \frac{\eta^{\beta-2}}{4-\beta} \frac{2}{\beta-3}.
\end{align*}
Combining it all allows to conclude.
\end{proof}

\begin{proof}[Proof of Theorem \ref{th:alpha_epsilon_uniform_power}]
The theorem is the straightforward combination of the two previous lemmas.
\end{proof}

\section{Power laws}
\label{app:power}

\subsection{Useful lemmas}

\begin{lemma}
Assume $\nu \notin \{1,2,3\}$. Let $J(m,\kappa,\nu) = \int_{1/m}^{1/2} (1-h)^{-\kappa} h^{-\nu} dh$. Then 
\begin{equation*}
J(m,\kappa,\nu) = \frac{m^{\nu-1}}{\nu-1} + \frac{\kappa m^{\nu-2}}{\nu-2} 
+ O(1 + m^{\nu-3}).
\end{equation*}
\end{lemma}

\begin{proof}
First note that for $h$ close to 0, we have the following approximation:
\begin{equation*}
(1-h)^{-\kappa} = 1 + \kappa h + 
h^2 R(h),
\end{equation*}
where $R(h) = O(1)$. In other words, there exists a constant $C$ such that, for $h \in [0,1/2]$, we have $|R(h)| \leq C$. As a result, we can write
\begin{equation*}
J = \int_{1/m}^{1/2} h^{-\nu} dh + \kappa \int_{1/m}^{1/2} h^{1-\nu} dh 
+ J_R,
\end{equation*}
where $|J_R| \leq C \int_{1/m}^{1/2} h^{2-\nu} dh = C \left[ \frac{h^{3-\nu}}{|3-\nu|} \right]_{1/m}^{1/2} = O(1) + O(m^{\nu-3})$. By ignoring constant terms, we then obtain
\begin{equation*}
J = \frac{m^{\nu-1}}{\nu-1} + \frac{\kappa m^{\nu-2}}{\nu-2} 
+ O(1) + O(m^{\nu-3}).
\end{equation*}
This is what we wanted.
\end{proof}

Let us define $I(m,\kappa,\nu) = \int_1^{m-\eta} (m-g)^{-\kappa} g^{-\nu} dg$.

\begin{lemma}
Assume $\kappa,\nu \notin \{1,2\}$. As $m \rightarrow \infty$, we have
\begin{equation*}
I(m,\kappa,\nu) = \frac{m^{-\kappa}}{\nu-1} + \frac{\kappa m^{-\kappa-1}}{\nu-2}  + \frac{\eta^{1-\kappa} m^{-\nu}}{\kappa-1} + \frac{\nu \eta^{2-\kappa} m^{-\nu-1}}{\kappa-2} + O(m^{-\kappa-2} + m^{-\nu-2}).
\end{equation*}
\end{lemma}

\begin{proof}
We consider the change of variable $g=mh$. Observe that
\begin{align*}
I&(m,\kappa, \nu) = m^{1-\kappa-\nu} \int_{1/m}^{1-\eta/m} (1-h)^{-\kappa} h^{-\nu} dh \\
&= m^{1-\kappa-\nu} \left( \int_{1/m}^{1/2} (1-h)^{-\kappa} h^{-\nu} dh + \int_{1/2}^{1-\eta/m} (1-h)^{-\kappa} h^{-\nu} dh \right) \\
&= m^{1-\kappa-\nu} \left( \int_{1/m}^{1/2} (1-h)^{-\kappa} h^{-\nu} dh + \int_{\eta/m}^{1/2} h^{-\kappa} (1-h)^{-\nu} dh \right) \\
&= m^{1-\kappa-\nu} \left( J(m,\kappa,\nu) + J(m/\eta, \nu,\kappa) \right) \\
&= \frac{m^{-\kappa}}{\nu-1} + \frac{\kappa m^{-\kappa-1}}{\nu-2}  + \frac{\eta^{1-\kappa} m^{-\nu}}{\kappa-1} + \frac{\nu \eta^{2-\kappa} m^{-\nu-1}}{\kappa-2} + O(m^{-\kappa-2} + m^{-\nu-2}),
\end{align*}
which concludes the proof.
\end{proof}

\subsection{Proving Theorem \ref{th:power_power}}

\begin{lemma}
Assume $b<\beta-1$ and $c<\gamma-1$. We have the following equality:
\begin{align*}
\alpha \mathbb E_\alpha[G^c\xi^b] = (\gamma-1) &\frac{\beta-1}{\beta-1-b} \eta^{\beta-1} I(m,\beta-1-b,\gamma-c) \nonumber \\ 
&+ \frac{\gamma-1}{\gamma-1-c} \frac{\beta-1}{\beta-1-b} \eta^b (m-\eta)^{1+c-\gamma}.
\end{align*}
Note that this is equal to $\mathbb P[M \geq m_\alpha]$ for $c=b=0$.
\end{lemma}

\begin{proof}
We separate the integration into $\alpha \mathbb E_\alpha[G^c \xi^b] = A + B$, where $A$ takes care of the case $g \leq m-\eta$, and $B$ focuses on $g \geq m-\eta$. We then have
\begin{align*}
A &= (\gamma-1)(\beta-1) \eta^{\beta-1} \int_{g=1}^{m-\eta} \int_{m-g}^\infty x^{b-\beta} dx g^{c-\gamma} dg \\
&= (\gamma-1)(\beta-1) \eta^{\beta-1} \int_{g=1}^{m-\eta} \left[ \frac{x^{1+b-\beta}}{1-\beta} \right]_{m-g}^\infty g^{c-\gamma} dg \\
&= (\gamma-1) \frac{\beta-1}{\beta-1-b} \eta^{\beta-1} I(m,\beta-1-b,\gamma-c).
\end{align*}
Now, we need to compute $B$, yielding
\begin{align*}
B &= (\gamma-1)(\beta-1) \eta^{\beta-1} \int_{m-\eta}^{\infty} \int_{\eta}^\infty x^{b-\beta} dx g^{c-\gamma} dg \\
&= (\gamma-1)(\beta-1) \eta^{\beta-1} \int_{m-\eta}^{\infty} \frac{\eta^{1+b-\beta}}{\beta-1-b} g^{c-\gamma} dg \\
&= (\gamma-1) \frac{\beta-1}{\beta-1-b} \eta^b \left[ \frac{g^{1+c-\gamma}}{1+c-\gamma} \right]_{m-\eta}^{\infty} \\
&= \frac{\gamma-1}{\gamma-1-c} \frac{\beta-1}{\beta-1-b} \eta^b (m-\eta)^{1+c-\gamma}. 
\end{align*}
Adding $A$ and $B$ yields the lemma.
\end{proof}

\begin{lemma}
Assume $b<\beta-1$ and $c<\gamma-1$. We then have the following asymptotic approximation:
\begin{equation*}
\mathbb E_\alpha[G^c\xi^b] \sim \frac{\gamma-1}{\gamma-1-c} \frac{\beta-1}{\beta-1-b}  \frac{\eta^b m^{1+c-\gamma} + \eta^{\beta-1} m^{1+b-\beta} }{m^{1-\gamma} + \eta^{\beta-1} m^{1-\beta}}.
\end{equation*}
\end{lemma}

\begin{proof}
This is derived from the two previous lemmas. Note in particular that only the leading term of $I(m,\beta-1-b,\gamma-c)$ is nonnegligible compared to $m^{b-\beta} + m^{c-\gamma}$.
\end{proof}

Note that we will need a more sophisticated approximation ratio for the case $\gamma-\beta > 1$. In this case, the approximation of the denominator can be slightly simplified, which yields
\begin{align}
\mathbb E_\alpha[G^c \xi^b] =& \frac{\gamma-1}{\gamma-1-c}  \frac{\beta-1}{\beta-1-b} \left( 1 - \frac{(\gamma-1)(\beta-1)}{\gamma-2} m^{-1} + O(m^{-2}) \right) \nonumber \\
& + \Big( \eta^{1+b-\beta} m^{1+c-(\gamma-\beta)} - \eta^{1+b-\beta} (1+c-\gamma) m^{c-(\gamma-\beta)} \nonumber\\
&\quad + m^b + \frac{(\gamma-1-c)(\beta-1-b)}{\gamma-2-c} m^{b-1} + O\left(m^{c-1-(\gamma-\beta)} + m^{b-2}\right) \Big). 
\label{eq:long_asymptotic_expectation}
\end{align}

\begin{proof}[Proof of Theorem \ref{th:power_power}]
To derive simpler forms of these asymptotic approximations, we are going to need to distinguish numerous cases, depending on the value of $\gamma-\beta$. The different cases are $\{ (-\infty,-2),  (-2,-1), (-1,0), (0,1), (1,2), (2,+\infty) \}$.

We have the following approximations:
\begin{align}
\alpha &\sim \left\{
\begin{array}{cc}
m^{1-\gamma}, & \text{ if } \gamma-\beta < 0  \\
\eta^{\beta-1} m^{1-\beta}, & \text{ if } \gamma-\beta > 0 \label{eq:alpha_power_power} \\
\end{array}
\right. \\
\mathbb E_\alpha[G] &\sim \left\{
\begin{array}{cc}
\frac{\gamma-1}{\gamma-2} m, & \text{ if } \gamma-\beta < 0  \\
\frac{\gamma-1}{\gamma-2} \eta^{1-\beta} m^{1-(\gamma-\beta)}, & \text{ if } \gamma-\beta \in (0,1)  \\
\frac{\gamma-1}{\gamma-2}, & \text{ if } \gamma-\beta > 1  \\
\end{array}
\right. \\
\mathbb E_\alpha[\xi] &\sim \left\{
\begin{array}{cc}
\frac{\beta-1}{\beta-2} \eta, & \text{ if } \gamma-\beta < -1  \\
\frac{\beta-1}{\beta-2} \eta^{\beta-1} m^{1+\gamma-\beta}, & \text{ if } \gamma-\beta \in (-1,0)  \\
\frac{\beta-1}{\beta-2} m, & \text{ if } \gamma-\beta > 0  \\
\end{array}
\right. \\
\mathbb E_\alpha[G^2] &\sim \left\{
\begin{array}{cc}
\frac{\gamma-1}{\gamma-3} m^2, & \text{ if } \gamma-\beta < 0 \\
\frac{\gamma-1}{\gamma-3} \eta^{1-\beta} m^{2-(\gamma-\beta)}, & \text{ if } \gamma-\beta \in (0,2)  \\
\frac{\gamma-1}{\gamma-3}, & \text{ if } \gamma-\beta > 2  \\
\end{array}
\right. \\
\mathbb E_\alpha[\xi^2] &\sim \left\{
\begin{array}{cc}
\frac{\beta-1}{\beta-3} \eta^2, & \text{ if } \gamma-\beta < -2  \\
\frac{\beta-1}{\beta-3} \eta^{\beta-1} m^{2+\gamma-\beta}, & \text{ if } \gamma-\beta \in (-2,0)  \\
\frac{\beta-1}{\beta-3} m^2, & \text{ if } \gamma-\beta > 0  \\
\end{array}
\right. \\
\mathbb E_\alpha[G\xi] &\sim \left\{
\begin{array}{cc}
\frac{\gamma-1}{\gamma-2} \frac{\beta-1}{\beta-2} \eta m, & \text{ if } \gamma-\beta < 0  \\
\frac{\gamma-1}{\gamma-2} \frac{\beta-1}{\beta-2} m, & \text{ if } \gamma-\beta > 0  \\
\end{array}
\right. 
\end{align}
In the case $\gamma - \beta > 1$, we are going to need the following more precise approximations, derived from Equation \ref{eq:long_asymptotic_expectation}. This yields
\begin{align*}
\mathbb E_\alpha[\xi] &= \frac{\beta-1}{\beta-2} m - \frac{\beta-1}{\beta-2} \frac{\gamma-1}{\gamma-2} + o(1), \\
\mathbb E_\alpha[G\xi] &= \frac{\gamma-1}{\gamma-2} \frac{\beta-1}{\beta-2} m - \frac{(\gamma-1)^2}{(\gamma-2)(\gamma-3)} \frac{\beta-1}{\beta-2} + o(1), \\
\mathbb E_\alpha[G] &= \left\{
\begin{array}{cc}
\frac{\gamma-1}{\gamma-2} + \frac{\gamma-1}{\gamma-2} \eta^{1-\beta} m^{3-(\gamma-\beta)} + O(m^{-1}), & \text{ if } \gamma-\beta \in (2,3) \\
\frac{\gamma-1}{\gamma-2} + \frac{(\gamma-1)(\beta-1)}{(\gamma-2)^2(\gamma-3)} m^{-1} + o(m^{-1}), & \text{ if } \gamma-\beta > 3  \\
\end{array}
\right. \\
\mathbb E_\alpha[G^2] &=\left\{
\begin{array}{cc}
\frac{\gamma-1}{\gamma-3} \eta^{1-\beta} m^{3-(\gamma-\beta)} + \frac{\gamma-1}{\gamma-3} +o(1), & \text{ if } \gamma-\beta \in (2,3) \\
\frac{\gamma-1}{\gamma-3} + \frac{\gamma-1}{\gamma-3} \eta^{1-\beta} m^{3-(\gamma-\beta)} + O(m^{-1}), & \text{ if } \gamma-\beta \in (3,4) \\
\frac{\gamma-1}{\gamma-3} + \frac{2(\gamma-1)(\beta-1)}{(\gamma-2)(\gamma-3)(\gamma-4)} + o(m^{-1}), & \text{ if } \gamma-\beta > 3  \\
\end{array}
\right. 
\end{align*}

As a result we have:
\begin{align*}
Var_\alpha(G) &\sim \left\{
\begin{array}{cc}
\frac{\gamma-1}{(\gamma-2)^2(\gamma-3)} m^2, & \text{ if } \gamma-\beta < 0 \\
\frac{\gamma-1}{\gamma-3} \eta^{1-\beta} m^{2-(\gamma-\beta)}, & \text{ if } \gamma-\beta \in (0,2) \\
\frac{\gamma-1}{(\gamma-2)^2(\gamma-3)}, & \text{ if } \gamma-\beta > 2  \\
\end{array}
\right. \\
Var_\alpha(\xi) &\sim \left\{
\begin{array}{cc}
\frac{\beta-1}{(\beta-2)^2(\beta-3)}, & \text{ if } \gamma-\beta < -2 \\
\frac{\beta-1}{\beta-3} \eta^{\beta-1} m^{2+\gamma-\beta}, & \text{ if } \gamma-\beta \in (-2,0) \\
\frac{\beta-1}{(\beta-2)^2(\beta-3)} m^2, & \text{ if } \gamma-\beta > 0  \\
\end{array}
\right. \\
Cov_\alpha(G, \xi) &\sim \left\{
\begin{array}{cc}
o(m), & \text{ if } \gamma-\beta < -1  \\
- \frac{\gamma-1}{\gamma-2} \frac{\beta-1}{\beta-2} \eta^{\beta-1} m^{2+\gamma-\beta}, & \text{ if } \gamma-\beta \in (-1,0)  \\
- \frac{\gamma-1}{\gamma-2} \frac{\beta-1}{\beta-2} \eta^{1-\beta} m^{2-(\gamma-\beta)}, & \text{ if } \gamma-\beta \in (0,1) \\
o(m), & \text{ if } \gamma-\beta > 1 
\end{array}
\right. 
\end{align*}
This results in:
\begin{align*}
Cov_\alpha(G, M) &\sim \left\{
\begin{array}{cc}
\frac{\gamma-1}{(\gamma-2)^2(\gamma-3)} m^2, & \text{ if } \gamma-\beta < 0  \\
\left( \frac{\gamma-1}{\gamma-3} - \frac{\gamma-1}{\gamma-2} \frac{\beta-1}{\beta-2} \right) \eta^{1-\beta} m^{2-(\gamma-\beta)}, & \text{ if } \gamma-\beta \in (0,2) \\
- \frac{(\gamma-1) \left( 1+(\beta-1)(\gamma+\beta-3) \right)}{(\gamma-2)^2(\gamma-3)(\beta-2)}, & \text{ if } \gamma-\beta > 2.
\end{array}
\right. \\
Var_\alpha(M) &\sim \left\{
\begin{array}{cc}
\frac{\gamma-1}{(\gamma-2)^2(\gamma-3)} m^2, & \text{ if } \gamma-\beta < 0  \\
\frac{\beta-1}{(\beta-2)^2(\beta-3)} m^2, & \text{ if } \gamma-\beta > 0.
\end{array}
\right.
\end{align*}
Finally, we obtain
\begin{equation*}
\rho_\alpha \sim \left\{
\begin{array}{cc}
1, & \text{ if } \gamma-\beta < 0  \\
- C \eta^{\frac{1-\beta}{2}} m^{-\frac{\gamma-\beta}{2}}, & \text{ if } \gamma-\beta \in (0,2), \\
- D m^{-1}, & \text{ if } \gamma-\beta > 2.
\end{array}
\right.
\end{equation*}
where
\begin{align*}
C &= (\beta-2) \left( \frac{\gamma-1}{\gamma-2} \frac{\beta-1}{\beta-2}  - \frac{\gamma-1}{\gamma-3}  \right) \sqrt{\frac{(\gamma-3)(\beta-3)}{(\gamma-1)(\beta-1)}} >0, \\
D &= \frac{1+(\beta-1)(\gamma+\beta-3)}{\gamma-2} \sqrt{\frac{(\gamma-1)(\beta-3)}{(\gamma-3)(\beta-1)}} >0.
\end{align*}
It is clear that $D$ is positive. To verify this for $C$, note that for $\gamma > \beta+1$, the function $\nu \mapsto \frac{\nu-1}{\nu-2}$ is decreasing for $\nu > 2$, and thus $\frac{\gamma-1}{\gamma-2} \frac{\beta-1}{\beta-2} > \frac{\gamma-1}{\gamma-2} \frac{\gamma-2}{\gamma-3} = \frac{\gamma-1}{\gamma-3}$. Replacing $m$ by its approximation derived from Equation \ref{eq:alpha_power_power}, replacing $\eta$ by its expression in terms of $\beta$, $\gamma$ and $\varepsilon$, and fitting the extra $\beta$ and $\gamma$ terms into the constants then yields the theorem.
\end{proof}

\section{Uniform goal, log-normal noise}
\label{app:lognormal}

We ran computations using the functions  $\mathrm{special.erf}$ and $\mathrm{integrate.quad}$ of python, by noting that $\mathbb P[\xi \geq m-g] = \mathbb P[\ln \xi \geq \ln(m-g)] =  \frac{1}{2}- \frac{1}{2} \mathrm{erf}\left(\frac{\ln(m-g)}{\eta}\right)$. But the results seemed sometimes noisy, because of computational approximations of very small numbers. We tested their robustness by also computing discrete approximations of the integrals, which boils down to the case where $G$ is a uniform distribution over a finite set and equally spaced points in $[0,1]$. The results were similar.

We also analyzed the curves obtained for other values of $\varepsilon$. For smaller values of $\varepsilon$, like $\varepsilon = 1/16$ in Figure \ref{fig:uniform_lognormal_epsilon16} and \ref{fig:uniform_lognormal_epsilon16_zoom}, the curves obtained still featured a double descent-like phenomenon, though not quite. Indeed, the goal increased, but instead of going then down, it only slowed down its increase. Eventually, the increase was rapid again.

\begin{figure}[!ht]
        \includegraphics[width=0.5\textwidth]{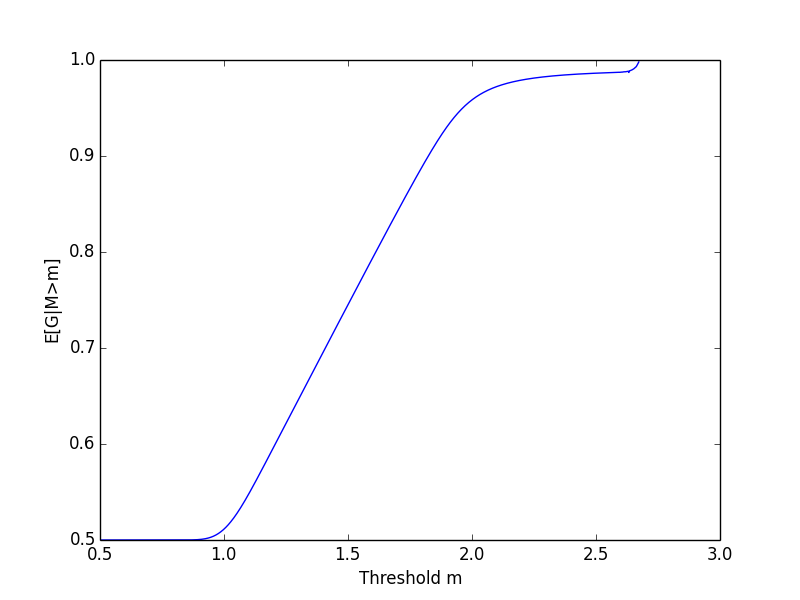}         \includegraphics[width=0.5\textwidth]{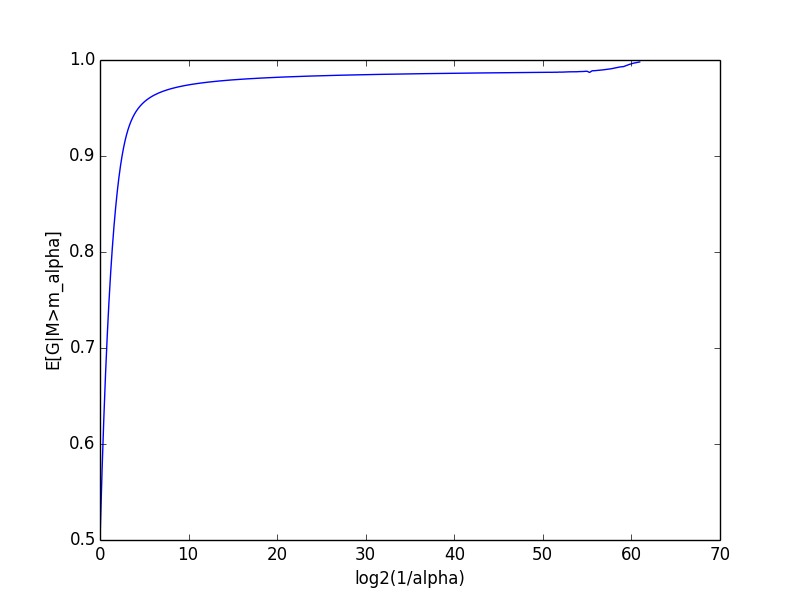} 
\caption{As $\alpha$ decreases, the measure threshold $m_\alpha$ increases. The figure illustrates this phenomenon as a function of the measure threshold (on the left) and as a function of $\log_2(1/\alpha)$ (on the right). The parameters were set at $\varepsilon = 1/16$.}
\label{fig:uniform_lognormal_epsilon16}
\end{figure}

\begin{figure}[!ht]
        \includegraphics[width=0.5\textwidth]{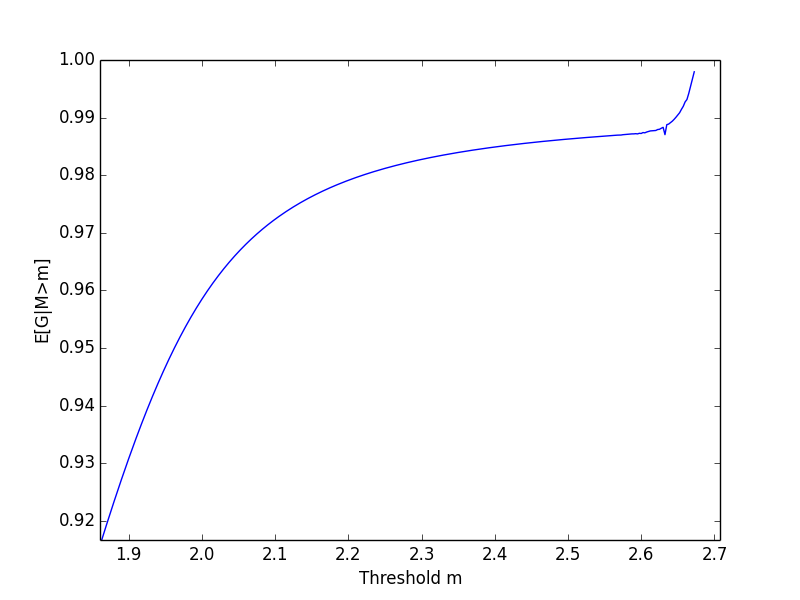} 
        \includegraphics[width=0.5\textwidth]{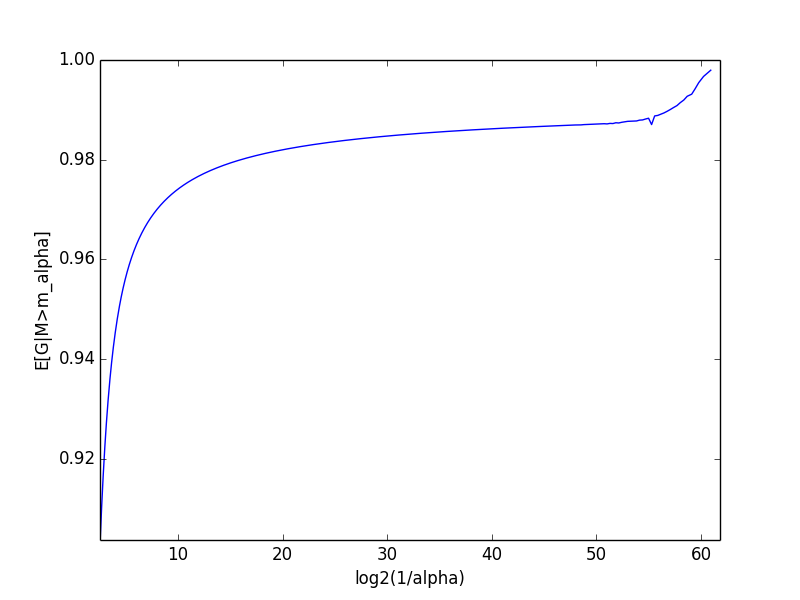} 
\caption{A zooming-in of Figure \ref{fig:uniform_lognormal_epsilon16}.}
\label{fig:uniform_lognormal_epsilon16_zoom}
\end{figure}

This observation with the fact that for small values of $\varepsilon$, the log-normal distributions become more similar to the normal distribution. 

For larger values of $\varepsilon$, the double descent phenomenon seems more pronounced, though our observations seem limited by approximation errors of the computations. We provide below the graphs for $\varepsilon = 1/2$, using the function $\mathrm{integrate.quad}$, and the by discretizing the values of $G$.

\begin{figure}[!ht]
        \includegraphics[width=0.5\textwidth]{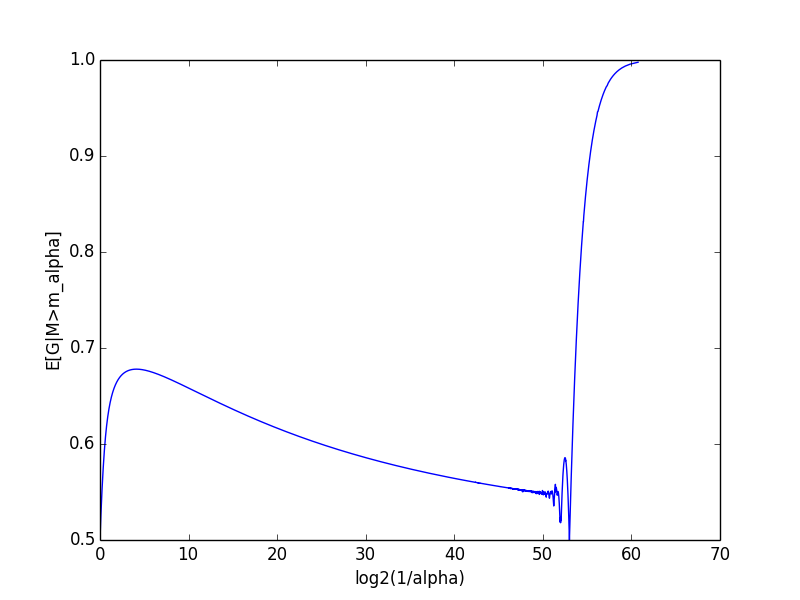}
        \includegraphics[width=0.5\textwidth]{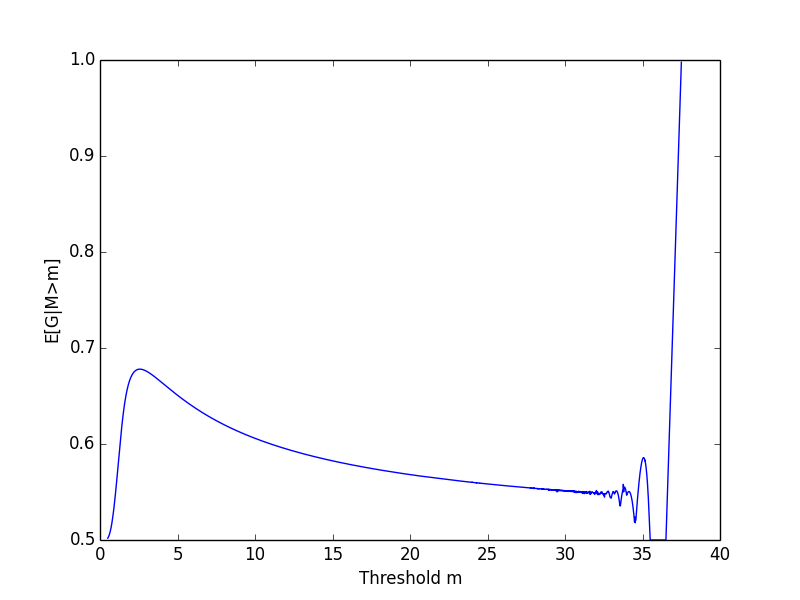} 
\caption{As $\alpha$ decreases, the measure threshold $m_\alpha$ increases. The figure illustrates this phenomenon as a function of the measure threshold (on the left) and as a function of $\log_2(1/\alpha)$ (on the right). The parameters were set at $\varepsilon = 1/2$.}
\label{fig:uniform_lognormal_epsilon2}
\end{figure}

\begin{figure}[!ht]
        \includegraphics[width=0.5\textwidth]{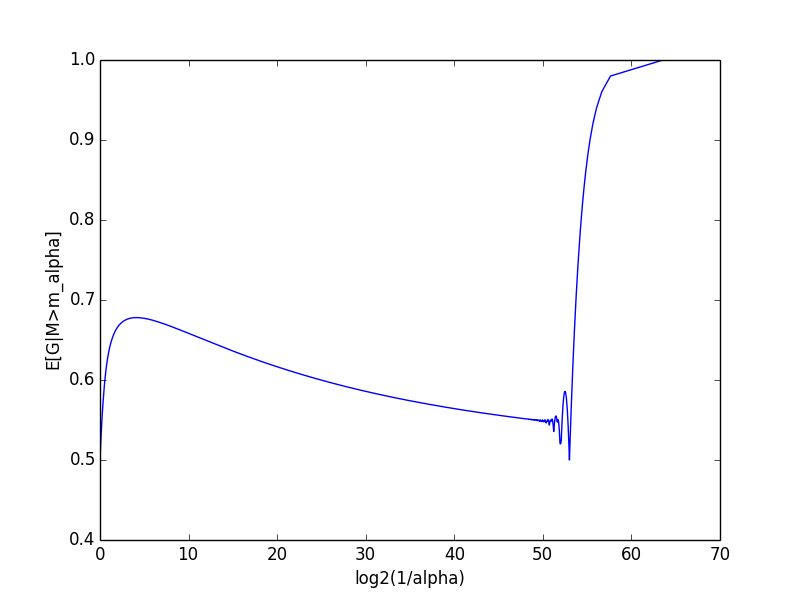} 
        \includegraphics[width=0.5\textwidth]{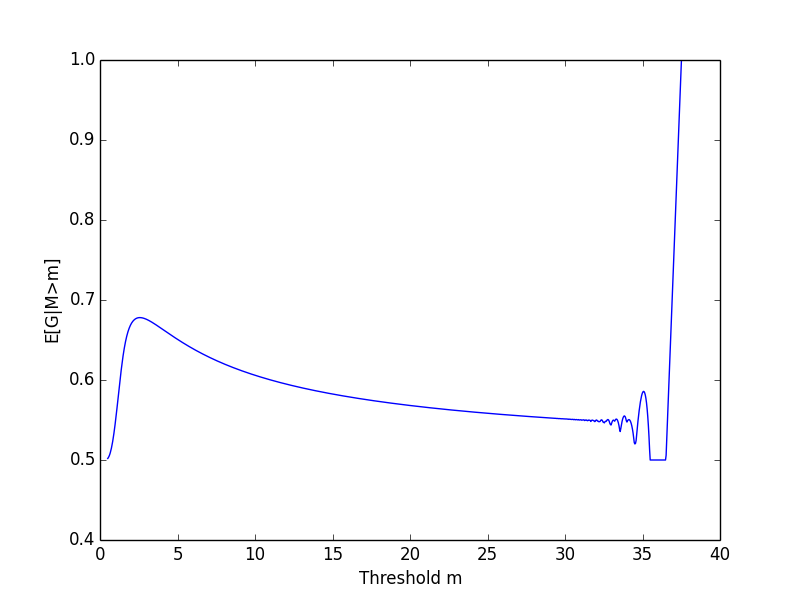}
\caption{Double descent for $\varepsilon = 1/2$, assuming a uniform distribution of $G$ on the finite set $\{ (k-1)/(n-1) | k \in [n] \}$, for $n=10,000$.}
\label{fig:uniform_lognormal_epsilon2_zoom}
\end{figure}

\section{Worst-case Goodhart}
\label{app:worst_case}

\begin{proof}
To prove this worrying result, we obviously need a goal that can take arbitrarily bad values. We shall assume that $G$ takes negative values only, and has a probability density function $p^G(g) = e^g$ for all $g \leq 0$. 

Now, the trick to derive the Theorem is to consider a power law noise $\xi$ whose tail is thicker for large negative values of $g$. More precisely, we shall consider that $p^\xi_2(x|g) = (\beta_g-1) \eta_g^{\beta_g-1} x^{-\beta_g}$, where $\beta_g = 4+\frac{1}{1-g}$. Unfortunately, such a distribution does not have zero mean. We correct this by adding a mass distribution $\delta_{x_g}$ on some point $x_g <0$. The density of $\xi$ given $G=g$ will then be given by $p^\xi(x|g) = \frac{1}{2} \delta_{x_g} + p^\xi_2(x|g)$. 

We can then fit $x_g$ and $\eta_g$ to guarantee zero mean and $\varepsilon^2$ variance for any value of $g$ without affecting the tail distribution, by choosing
\begin{equation*}
\eta_g^2 = \frac{2 \varepsilon^2}{ \left(\frac{\beta_g-1}{\beta_g-2}\right)^2 + \frac{\beta_g-1}{\beta_g-3}} \quad \text{and} \quad x_g = - \frac{2(\beta_g-1) \eta_g}{\beta_g-2}
\end{equation*}

We note that $4 \leq \beta_g \leq 5$. As a result, we have $\varepsilon / 2 \leq \eta_g \leq 2 \varepsilon /\sqrt{5} \leq \varepsilon$. Thus, for $m \geq \varepsilon$, we have
\begin{align*}
  P(m,g) &\triangleq \mathbb P[\xi \geq m-g|G=g] = \frac{1}{2} \int_{m-g}^\infty (\beta_g-1) \eta_g^{\beta_g-1} x^{-\beta_g} \\ &= \frac{1}{2} \left( \frac{\eta_g}{m-g} \right)^{3+\frac{1}{1-g}}.
\end{align*}
Using the bounds above, we observe that
\begin{equation*}
P(m,g) \leq \left( \frac{\varepsilon}{m} \right)^{3+\frac{1}{1-g}} \leq \varepsilon^3 m^{-3} \exp\left( - \frac{\ln m}{1-g} \right).
\end{equation*}
Moreover, for $g \geq -m$, we have
\begin{equation*}
P(m,g) \geq \left( \frac{\varepsilon}{4m} \right)^{3+\frac{1}{1-g}} \geq \frac{\varepsilon^4}{256} m^{-3} \exp\left( - \frac{\ln m}{1-g} \right).
\end{equation*}
We can now upper-bound the value of $G$ assuming $M \geq m$, which yields
\begin{align*}
  \mathbb E[G|M=m] &= \frac{\int_{-\infty}^0 P(m,g) e^g g dg}{\int_{-\infty}^0 P(m,g) e^g dg} \leq \frac{\int_{-m}^0 P(m,g) e^g g dg}{\int_{-\infty}^0 P(m,g) e^g dg} \\
  &\leq - \frac{\varepsilon}{256} \frac{\int_0^m \exp(-\frac{\ln m}{1+h}-h) h dh}{\int_0^\infty \exp(-\frac{\ln m}{1+h}-h) dh}
\end{align*}
To simplify notation, let us denote $\ell = \sqrt{\ln m}$. We then write the integral in the denominator as
\begin{equation*}
  \int_0^\infty \exp\left(-\frac{\ell^2}{1+h}-h\right) dh = I(m) + J(m) + K(m),
\end{equation*}
where $I(m)$ is the integral from 0 to $\frac{\ell}{5}-1$, $J(m)$ is the one from $\frac{\ell}{5}-1$ to $5\ell$, and $K(m)$ is the one from $5\ell$ to infinity. Intuitively, $I(m)$ and $K(m)$ will be negligible, so that we can only focus on $J(m)$. To prove this formally, we use basic bounds:
\begin{align*}
  I(m) &\leq \int_0^{\frac{\ell}{5}-1} \exp\left(-\frac{\ell^2}{\ell/5}\right) dh \leq \left(\frac{\ell}{5}-1\right) e^{-5\ell}, \\
  J(m) &\geq \int_{\ell/2-1}^{2\ell} \exp\left( -2\ell - 2\ell \right) dh \geq \ell e^{-4\ell}, \\
  K(m) &\leq \int_{5\ell}^\infty \exp(-h) dh = e^{-5\ell},
\end{align*}
In particular, we see that $\frac{I(m)+K(m)}{J(m)} \leq e^{-\ell} \leq 1$. Combining it all, we see that
\begin{align*}
  \mathbb E[G|M=m] &\leq - \frac{\varepsilon}{256} \frac{\int_{\ell/5-1}^{5\ell} \exp(-\frac{\ell^2}{1+h}-h) h dh}{I(m)+J(m)+K(m)} \\
  &\leq - \frac{\varepsilon}{256} \frac{(\ell/5-1) J(m)}{I(m)+J(m)+K(m)} \\
  &\leq - \frac{\varepsilon}{256} (\ell/5-1) \frac{1}{1+\frac{I(m)+K(m)}{J(m)}} \\
  &\leq - \Omega(\varepsilon \sqrt{\ln m}).
\end{align*}
As we take the limit $\alpha \rightarrow 0$, we have $m_\alpha \rightarrow \infty$, which then implies $\mathbb E_\alpha[G] = \mathbb E[G|M=m_\alpha] \rightarrow - \infty$.
\end{proof}

\end{appendix}

\end{document}